\documentclass[journal]{IEEEtran}
\usepackage{amsmath,amsfonts}
\usepackage{algorithmic}
\usepackage{algorithm}
\usepackage{array}
\usepackage[caption=false,font=normalsize,labelfont=sf,textfont=sf]{subfig}
\usepackage{textcomp}
\usepackage{stfloats}
\usepackage{url}
\usepackage{verbatim}
\usepackage{graphicx}
\usepackage{cite}
\usepackage{amsmath}
\usepackage{amssymb}
\usepackage{mathtools}
\usepackage{amsthm}

\usepackage{booktabs}  
\usepackage{multirow}   
\usepackage{graphicx}   
\usepackage[table]{xcolor} 

\usepackage[capitalize,noabbrev]{cleveref}

\def\eg{\emph{e.g.}} 

\def\ie{\emph{i.e.}} 

\def\etc{\emph{etc.}} 
\def\vs{\emph{vs. }}

\theoremstyle{plain}
\newtheorem{theorem}{Theorem}[section]
\newtheorem{proposition}[theorem]{Proposition}

\theoremstyle{definition}
\newtheorem{definition}[theorem]{Definition}
\newtheorem{assumption}[theorem]{Assumption}
\theoremstyle{remark}
\newtheorem{remark}[theorem]{Remark}

\begin{document}

\title{Detecting Deepfakes via Hamiltonian Dynamics}

\author{Harry Cheng,~\IEEEmembership{Member,~IEEE},
        Ming-Hui Liu,
        Tianyi Wang,~\IEEEmembership{Member,~IEEE},
        Weili Guan,~\IEEEmembership{Member,~IEEE},
        Liqiang Nie,~\IEEEmembership{Senior Member,~IEEE},
        Mohan Kankanhalli,~\IEEEmembership{Fellow, IEEE}

\thanks{
Harry Cheng, Tianyi Wang, and Mohan Kankanhalli are with the School of Computing, National University of Singapore, Singapore (e-mail: xaCheng1996@gmail.com; terry.ai.wang@gmail.com, dcsmsk@nus.edu.sg).
Ming-Hui Liu is with the School of Software, Shandong University, China (e-mail: liuminghui@mail.sdu.edu.cn). 
Weili Guan is with the School of Information Science and Technology, Harbin Institute of Technology (Shenzhen), China (e-mail: honeyguan@gmail.com).
Liqiang Nie is with the School of Computer Science and Technology, Harbin Institute of Technology (Shenzhen), Shenzhen, China (e-mail: nieliqiang@gmail.com).
}
}

\maketitle

\begin{abstract}
Driven by the rapid development of generative AI models, deepfake detectors are compelled to undergo periodic recalibration to capture newly developed synthetic artifacts.
To break this cycle, we propose a new perspective on deepfake detection: moving from static pattern recognition to dynamical stability analysis.
Specifically, our approach is motivated by physics-inspired priors: we hypothesize that natural images, as products of dissipative physical processes, tend to settle near stable, low-energy equilibria. In contrast, generative models optimize for statistical similarity to real images but do not explicitly enforce structural constraints such as geometric smoothness, leaving deepfakes more likely to occupy unstable, high-energy states.
To operationalize this, we introduce Hamiltonian Action Anomaly Detection (HAAD), comprising three contributions:
\textbf{i)} We model the image latent manifold as a potential energy surface. Under this hypothesis, real images are expected to produce basin-like low-energy responses, whereas fake images are more likely to induce high-potential, high-gradient responses.
\textbf{ii)} We employ Hamiltonian-inspired dynamics as a stability probe. By releasing latent states from rest, samples near stable regions remain bounded, while high-gradient samples produce larger trajectory responses.
\textbf{iii)} We quantify these dynamic behaviors through two trajectory statistics, \ie, Hamiltonian action and energy dissipation.
Extensive experiments show that HAAD outperforms evaluated state-of-the-art baselines on challenging cross-dataset transfer benchmarks, supporting a physics-inspired stability prior for digital forensics.
\end{abstract}


\section{Introduction}
\label{sec:intro}
The rapid advancement of generative models has blurred the boundary between reality and fabrication~\cite{GAN, SD3}. While fueling creativity, these techniques have lowered the barrier for malicious exploitation, such as `deepfakes' that threaten information ecosystems~\cite{Face2Face}. As a result, developing detection frameworks has been an urgent societal need.

The research community has proposed a number of detectors~\cite{cheng2024leavedeepfakedatatraining, Hong_Deepfake_CVPR_2024, cheng2025fair, ma2025specificity, xia2024mmnet}. 
The prevailing paradigm focuses on identifying \emph{specific static artifacts}, such as blending inconsistencies~\cite{SBI_ShioharaY22} or upsampling fingerprints~\cite{Tan_CVPR24}. 
However, these approaches suffer from a fundamental limitation: they tend to overfit to these artifacts~\cite{CVPR2025_1}. As generative architectures evolve, characteristic artifacts change drastically~\cite{DiFF}, rendering existing detectors ineffective and trapping the field in a `cat-and-mouse' cycle.

\begin{figure}[t]
    \centering
    \includegraphics[width=0.47\textwidth]{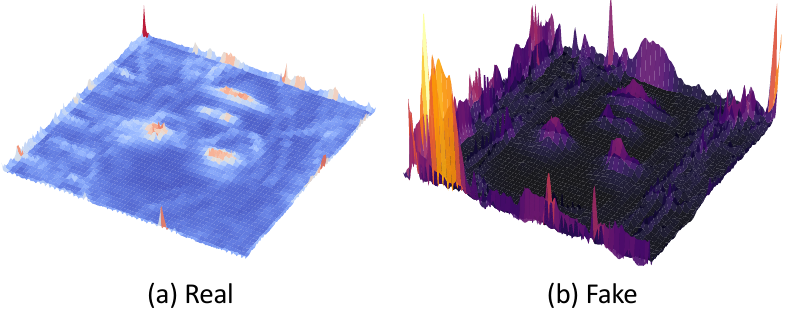}
    \caption{Visualization of microscopic structural irregularities, measured by Laplacian magnitude and averaged over 1,000 images from the Celeb-DF++ dataset~\cite{Celeb-DF++}. (a) Real images yield a smoother averaged manifold surface. (b) Forged images exhibit more jagged high-frequency spikes.}
    \label{fig:1_intro}
\end{figure}

To break this cycle, we seek a detection cue from the \emph{physics of image formation} rather than the specific artifacts of any particular generator, which remains valid as generative architectures evolve.
Specifically, real images are formed by physical imaging processes, \ie, they are captured by camera sensors through light reflected from surfaces, which inherently enforce local regularities such as geometric smoothness between neighboring pixels and photometric consistency of shading across surfaces. In contrast, fake images are produced by generative models (\eg, GANs and diffusion models) trained with distribution-level objectives that contain no explicit penalty for violating such local structure.
As a result, while generators can match natural images at the global statistical level, they often leave \emph{microscopic structural irregularities}~\cite{physics_as_prior} that physically formed real images do not exhibit.
As shown in \Cref{fig:1_intro}, real images form a smooth manifold surface, whereas deepfakes are punctuated by high-frequency jagged spikes.

To turn this observation into a quantitative signal, we draw an analogy with \emph{physical systems}. Specifically, in such systems, configurations that are smooth and relaxed (\eg, a flat elastic membrane) correspond to low-energy states, whereas locally stretched or jagged configurations store potential energy as elastic tension. That is, they reside in high-energy states.
By this analogy, we interpret the microscopic structure of the images in \Cref{fig:1_intro} as a form of \emph{energy}, where images exhibiting greater structural inconsistency (\ie, fake images) should correspond to higher energy values, while images satisfying the geometric and photometric regularities receive low energy.
It is worth noting that this energy is defined by physical regularities rather than by the specific artifacts of any particular generator. Therefore, the resulting real-vs-fake energy gap is less tied to any single generator and is expected to generalize across evaluated non-adaptive transfer settings.

However, directly comparing this static energy between real and fake images can be unreliable, as the gap may be subtle. To amplify this gap, we shift our focus from static measurement to \emph{dynamical evolution}. Specifically, we hypothesize that real images tend to reside in lower-energy states, while deepfakes are more likely to occupy higher-energy states. If we further conceptualize each image as a physical particle on an energy landscape, real particles will more often sit near stable basins, while fake particles will more often appear on steep slopes (as shown in \Cref{fig:energy}a). We term this assumption the \textbf{Manifold Stability Hypothesis} (MSH)\footnote{Please see \Cref{sec:Theory} for details.}.
To operationalize it, we model the particle's evolution within a \textbf{Hamiltonian mechanics} framework~\cite{cranmer2020lagrangian}. In this framework, the energy landscape exerts forces on the particle through its gradient, and we simulate the resulting trajectory. As illustrated in \Cref{fig:energy}b, particles near stable basins remain approximately stationary, whereas particles on steep slopes accelerate along the gradient. By tracking these trajectories, we convert subtle static energy differences into macroscopic motion patterns.

\begin{figure}[t]
    \centering
    \includegraphics[width=0.45\textwidth]{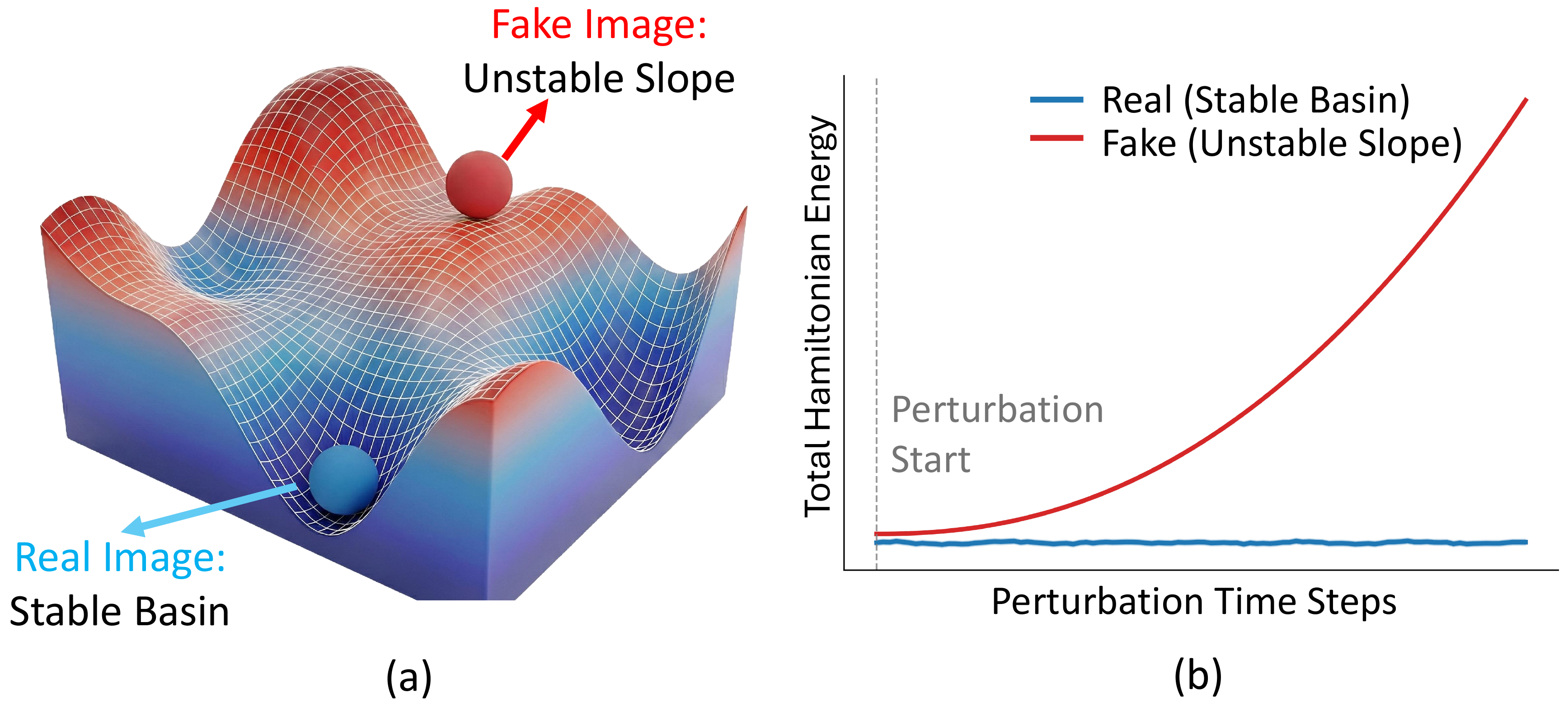}
    \caption{(a) Manifold Stability Hypothesis. Real images are expected to concentrate near stable basins (blue), while deepfakes are more likely to induce high-gradient slopes (red). (b) Hamiltonian-inspired trajectories. Basin-like samples remain bounded under perturbation, while high-gradient samples exhibit larger energy divergence (high action).}
    \label{fig:energy}
\end{figure}

To turn this hypothesis into a trainable stability probe, we propose Hamiltonian Action Anomaly Detection (HAAD). Specifically, HAAD consists of three components:
\textbf{i) Realizing the energy landscape} (\emph{learnable potential $V$}).
Under the MSH, the energy function should statistically separate basin-like real responses from high-gradient fake responses. We instantiate this landscape as a learnable potential $V$ built from two physics-motivated priors: geometric smoothness (via graph Laplacians, penalizing abrupt feature transitions between neighboring patches) and photometric consistency (via Lambertian shading constraints, penalizing implausible illumination patterns). Images that violate these priors receive high potential energy, making them more likely to generate slope-like trajectory responses.
\textbf{ii) Carrying out the dynamics} (\emph{symplectic evolution}).
We stipulate that each particle evolves under the force induced by $V$. To realize this, we employ a short-horizon symplectic integrator, which is designed to reduce solver-induced drift while probing the predicted behavior: particles near minima stay approximately put, while particles on slopes accelerate along the gradient. This can amplify small static energy gaps into larger trajectory discrepancies.
\textbf{iii) Reading out the motion patterns} (\emph{Action $S$ and Dissipation $D$}).
We distinguish real from fake by qualitative trajectory behavior via two statistics: \emph{Action} $S$, the accumulated energy cost measuring how forcefully the particle is driven, and \emph{Dissipation} $D$, the energy drift measuring deviation from stable evolution. Together, they serve as discriminative features for classification.
Through these three closely associated components, HAAD anchors deepfake detection on a stability-based inductive bias rather than transient artifacts.

We evaluate HAAD on several deepfake benchmarks~\cite{Celeb-DF, dfdc, dfd2019, DeepfakeBench}, where the cross-dataset and cross-generator experiments show strong transfer performance compared to several state-of-the-art (SoTA) baselines. Our contributions are three-fold:

\begin{itemize}
    \item We introduce a physics-inspired stability prior for deepfake detection. We posit that real images tend to reside near low-energy basins while deepfakes are more likely to occupy high-energy slopes of a physics-motivated energy landscape. This reframes detection from static pattern matching to \emph{dynamical stability analysis}.
    \item We propose \emph{Hamiltonian Action Anomaly Detection}, which learns a potential $V$ from physical priors, evolves image representations with a short-horizon symplectic integrator, and extracts two trajectory statistics as discriminative features for classification.
    \item Extensive experiments support the proposed MSH as an effective stability prior. HAAD achieves strong performance among SoTA baselines under cross-dataset and cross-generator experiments.
\end{itemize}

\section{Related Work}

\subsection{Deepfake and Synthetic Image Detection}
Our work addresses both deepfake manipulations (\eg, face swapping)~\cite{Face2Face, DiFF, diffswap} and general image synthesis (\eg, natural scenes)~\cite{GenImage, DIRE, SD3}. Across these domains, the research community is focused on generalization against unseen forgery methods.
Existing approaches typically fall into two paradigms: 
i) \textbf{Artifact Mining} focuses on identifying explicit statistical anomalies left by generative processes~\cite{xia2024advancing, liu2025learning}. Researchers have exploited frequency discrepancies~\cite{Exploring_Frequency_Adversarial, SSTNET}, blending boundary traces~\cite{SBI_ShioharaY22, x-ray}, and reconstruction errors~\cite{Face_Reconstruction, DIRE}.
ii) \textbf{Foundation Model Adaptation} employs the rich representations of pre-trained vision-language models~\cite{wu2023generalizable, liu2024forgery, ojha2023towards}.
For instance, VLFFD~\cite{VLFFD} utilizes CLIP~\cite{CLIP} to exploit semantic priors for detection.
Effort~\cite{Effort} integrates Singular Value Decomposition (SVD) into
pre-trained vision encoder to enhance sensitivity to structural artifacts.
Although these methods achieve impressive performance, they treat the input image as a \textit{static map}, classifying it based on learned features. This creates a `lag': as generative models evolve to suppress these specific artifacts, detectors become obsolete.
Some recent work goes beyond static appearance. For instance, FAMM~\cite{liao2023famm} detects compressed deepfake videos by analyzing facial muscle motion patterns rather than individual frames.
However, these motion-based approaches rely on temporal signals across video frames, which are unavailable for single-image deepfakes, and do not extend to general AIGI detection.
Our work occupies a complementary niche: we introduce a \emph{latent stability probe} based on Hamiltonian dynamics, which detects structural instability from a single image's feature representation without requiring temporal sequences or per-frame motion estimation.

\subsection{Physics-Informed Vision and Dynamics}
Unlike purely data-driven approaches, physics-based vision explicitly models the structural and photometric constraints in natural image formation~\cite{zhang2021physg}.
For instance, inverse rendering approaches demonstrate that real images adhere to photometric laws~\cite{li2020inverse}.
Parallel to this, physics-inspired neural networks incorporate physical laws as inductive biases for representation learning~\cite{physics_as_prior}, and Hamiltonian mechanics has been integrated into machine learning to model continuous dynamical systems~\cite{liu2020energy, cranmer2020lagrangian}.
Hamiltonian Neural Networks (HNNs)~\cite{hamiltonian_2019} enforce energy conservation to learn physical laws from observed trajectory data~\cite{duvenaud2020your}.
However, these methods are predominantly confined to \emph{simulation} or \emph{future forecasting} in closed environments.
To the best of our knowledge, our HAAD is the first work to cast deepfake detection as \emph{dynamical stability analysis} under a Hamiltonian framework.
Instead of learning \emph{static} physical consistency, our approach constructs a learnable potential energy surface based on physical priors and treats the image representation as a particle system.
By subjecting the image features to a symplectic evolution, we transform the detection problem into a \emph{dynamical stability test}, revealing the dynamical instability of synthesized images predicted by the Manifold Stability Hypothesis.
\section{Methodology} \label{sec:mot}

Before presenting the HAAD architecture, we first formalize the theoretical framework motivating its design. Specifically, we cast deepfake detection as \emph{dynamical stability analysis} on a learned potential energy landscape. Each latent feature is modeled as a dynamical state $(\mathbf{q}, \mathbf{p})$ that evolves under Hamiltonian mechanics, where $\mathbf{q}$ encodes image content and $\mathbf{p}$ is the canonical momentum driving the probe. Under this formulation, real images are hypothesized to reside near stable equilibria, since they are shaped by natural physical processes. By contrast, deepfakes, trained with objectives that do not explicitly enforce local structural regularities, tend to occupy unstable states on the landscape. We formalize this intuition as the \emph{Manifold Stability Hypothesis} (\Cref{ass:stability}) and derive detection cues from its predicted dynamical behavior.

\subsection{Theoretical Formulation}
\label{sec:Theory}

Directly verifying whether an image satisfies every physical regularity of natural scenes (surface smoothness, illumination coherence, geometric consistency, \etc) is infeasible: these regularities form an open-ended set, and any explicit enumeration risks missing cases by construction.
Therefore, HAAD does not aim for such completeness. Instead, we construct a \emph{potential energy landscape} based on a set of intuitive physical priors (detailed in \Cref{sec:gnn_potential}), where a point's elevation measures the total extent to which it violates those priors.

\vspace{1mm}
\noindent \textbf{Geometric Definitions of the Data Manifold.}
Let $E: \mathcal{X} \to \mathcal{Q} \subseteq \mathbb{R}^d$ be an encoder mapping the image space $\mathcal{X}$ to a latent manifold $\mathcal{Q}$. We denote $\mathbf{q} = E(\mathbf{x}) \in \mathcal{Q}$ as a latent position, where $\mathbf{x}$ is an input image. We interpret this manifold as a physical system governed by a scalar potential field.

\begin{definition}[\textbf{Latent Potential Energy}]
We define a potential function $V\!:\!\mathcal{Q}\!\to\!\mathbb{R}$ that measures how much a state $\mathbf{q}\!\in\!\mathcal{Q}$ deviates from the physics-motivated regularities of natural images.
Conceptually, this aligns with Energy-Based Models (EBMs), where the probability density of real data follows a Gibbs distribution $p(\mathbf{q})\!\propto\!\exp(-V(\mathbf{q}))$.
Within this framework, the \textbf{negative} gradient field $-\nabla V(\mathbf{q})$ represents the \textit{restorative force} driving the system toward high-density (low-energy) equilibrium regions.
\end{definition}

\noindent \textbf{Physical Motivation.}
Dissipative mechanics~\cite{zhong2019symplectic} and Lyapunov stability theory~\cite{hopfield1982neural} motivate the intuition that physical systems tend to lose energy until they settle near stable, low-energy equilibria. The Principle of Least Action provides a complementary perspective: through its static corollary (the Principle of Minimum Potential Energy), it characterises equilibrium states as potential minima (detailed in Appendix~\ref{app:theory}). Together, these principles motivate the hypothesis that real images, as products of natural dissipative processes, tend to reside near the `relaxed' basins of the manifold. In contrast, deepfakes are produced by generative models that optimize for statistical similarity to real images but do not explicitly enforce local structural constraints. As a result, synthesized samples may retain microscopic structural inconsistencies, which behave like `compressed springs' or particles trapped on steep slopes, `suspended' in high-potential states.

Based on this physical intuition, we introduce the following hypothesis on the topological states of real \vs fake data:

\begin{assumption}[\textbf{Manifold Stability Hypothesis}]
\label{ass:stability}
We hypothesize that natural images $\mathbf{x}_{real}$ are typically located near local minima (stable equilibria) of the learned potential surface $V$. Concretely, for a real latent representation $\mathbf{q}_{real}$, we model $\mathbf{q}_{real}$ as admitting a local neighborhood $\mathcal{N}$ such that:
\begin{equation}
    V(\mathbf{q}_{real}) \leq V(\mathbf{q}), \quad \forall \mathbf{q} \in \mathcal{N}.
\end{equation}

The gradient at equilibrium approximately vanishes: $\|\nabla V(\mathbf{q}_{real})\| \approx 0$.
In contrast, we expect deepfake samples $\mathbf{x}_{fake}$ to more frequently occupy high-gradient regions (steep slopes) of $V$, where $\|\nabla V(\mathbf{q}_{fake})\| \gg 0$\footnote{We provide a theoretical motivation linking dissipative mechanics and the Principle of Least Action to this stability hypothesis in Appendix~\ref{app:theory}. }.
\end{assumption}

\noindent \textbf{Scope of the Hypothesis.}
Assumption~\ref{ass:stability} is a \emph{distributional} modeling hypothesis about the populations of real and generated images under $V$, not a per-sample statement about any individual image. The asymmetry we exploit is a statistical tendency: current generative models are trained with distribution-level objectives that do not explicitly enforce local physical regularities, so their outputs are more likely to violate regularity and accumulate elevation under $V$. We thus treat $V$ as a \emph{probe} for inconsistency with these regularities. 

\vspace{1mm}

\noindent \textbf{Hamiltonian Dynamics as a Stability Probe.}
To distinguish the states defined in Assumption~\ref{ass:stability}, we construct a dynamical system to probe the manifold's geometry. We treat the latent feature $\mathbf{q}$ as a particle system.
To capture the varying sensitivity of different semantic features, we introduce a learnable state-conditioned diagonal preconditioner $\mathbf{M}(\mathbf{q})$.
In addition, we introduce a canonical momentum variable $\mathbf{p}$, initialized at zero ($\mathbf{p}_0 = \mathbf{0}$).

\begin{definition}[\textbf{Hamiltonian System}]
\label{def:hamilton_system}
The total energy of the system, denoted as the Hamiltonian $\mathcal{H}$, is the sum of kinetic energy $T(\mathbf{p})$ and potential energy $V(\mathbf{q})$:
\begin{equation}
    \mathcal{H}(\mathbf{q}, \mathbf{p}) = T(\mathbf{p}) + V(\mathbf{q}) = \frac{1}{2}\|\mathbf{p}\|^2 + V(\mathbf{q}).
\end{equation}

The temporal evolution of the feature state is governed by Hamilton's equations of motion:
\begin{equation}
    \label{eq:hamilton}
    \frac{d\mathbf{q}}{dt} = \frac{\partial \mathcal{H}}{\partial \mathbf{p}} = \mathbf{p}, \quad 
    \frac{d\mathbf{p}}{dt} = -\frac{\partial \mathcal{H}}{\partial \mathbf{q}} = -\nabla_{\mathbf{q}} V(\mathbf{q}).
\end{equation}
\end{definition}

\begin{figure*}[t]
    \centering
    \includegraphics[width=0.95\textwidth]{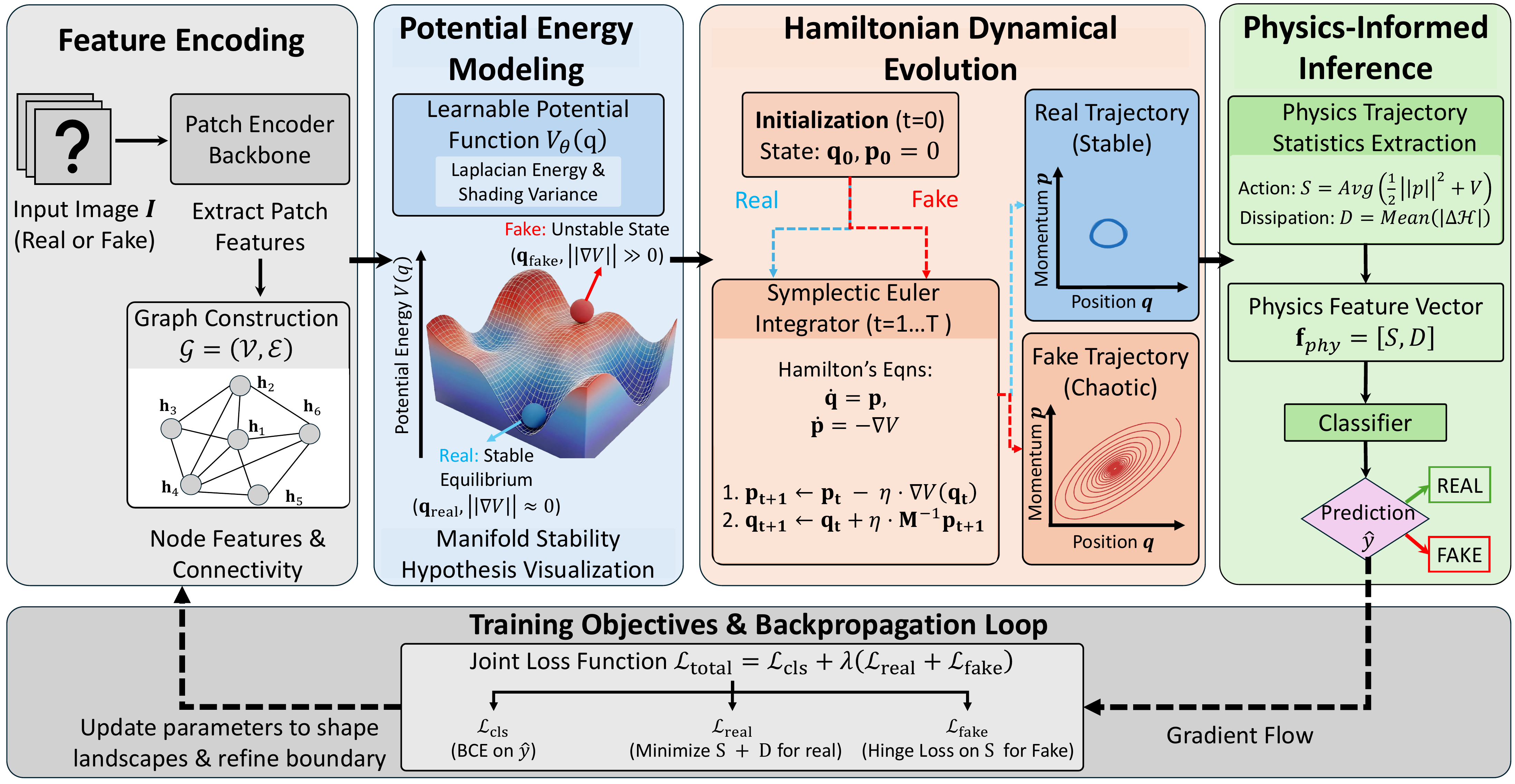}
    \caption{Overview of the proposed HAAD. We cast deepfake detection as \emph{dynamical stability analysis under a Hamiltonian framework}: \textbf{i)}~patch features are organized into a spatial graph, then fed into \textbf{ii)}~a learnable potential $V$ whose landscape encodes physical priors under the Manifold Stability Hypothesis, real samples (blue) are expected to fall near stable equilibria of $V$ while fakes (red) tend to occupy high-gradient unstable states. \textbf{iii)} A short-horizon symplectic integrator amplifies these differences into \textbf{iv)}~trajectory statistics $S$ and $D$ for classification, jointly trained with $\mathcal{L}_{cls}$ and $\mathcal{L}_{phy}$ (bottom loop).}
    \vspace{-1em}
    \label{fig:framework}
\end{figure*}

\noindent \textbf{Physical Interpretation.}
\Cref{eq:hamilton} implies that the \textbf{potential gradient} $-\nabla V$, which points toward states more consistent with the physics-motivated regularities, acts as a physical force. Since we initialize at rest ($\mathbf{p}_0=\mathbf{0}$), any movement is purely driven by this internal tension.
Under Assumption~\ref{ass:stability}, real images (where $\nabla V \approx 0$) are expected to remain approximately stationary, while deepfakes (where $\|\nabla V\| \gg 0$) are expected to experience immediate acceleration, converting potential energy into kinetic momentum.

\begin{proposition}[\textbf{Divergence of Hamiltonian Action}]
\label{prop:divergence}
Consider the evolution of a particle over a small time interval $t \in [0, \tau]$ starting from rest ($\mathbf{p}_0 = \mathbf{0}$). Under Assumption~\ref{ass:stability}, the kinetic energy $T$ for a deepfake sample exhibits quadratic growth over time, whereas for a real sample, it remains negligible. As a result, the accumulated Action $S$ satisfies:
\begin{equation}
    S(\mathbf{q}_{fake}) \gg S(\mathbf{q}_{real}).
\end{equation}
\end{proposition}

\begin{proof}[Proof Sketch]
We analyze the short-term dynamics via a first-order Taylor expansion of the potential around $\mathbf{q}_0$: $V(\mathbf{q}) \approx V(\mathbf{q}_0) + \nabla V(\mathbf{q}_0)^\top(\mathbf{q}-\mathbf{q}_0)$. 
Under Assumption~\ref{ass:stability}, a real image is expected to initialize near a local minimum where $\|\nabla V(\mathbf{q}_0)\| \approx 0$. Hamilton's equations imply momentum rate $\dot{\mathbf{p}} = -\nabla V \approx \mathbf{0}$. Therefore, the particle remains approximately stationary ($\mathbf{p}(t) \approx \mathbf{0}$), yielding minimal Action.

In contrast, a deepfake sample is expected to more often initialize on a slope with significant gradient $\|\nabla V(\mathbf{q}_0)\| = C > 0$. Within the small interval $\tau$, we approximate the gradient force as constant. Integrating $\dot{\mathbf{p}} \approx -C\mathbf{u}$ yields the momentum $\|\mathbf{p}(t)\| \approx C t$. Under the canonical kinetic energy $T(\mathbf{p}) = \tfrac{1}{2}\|\mathbf{p}\|^2$ of Definition~2, this grows quadratically:
\begin{equation}
    T(t) = \tfrac{1}{2}\|\mathbf{p}(t)\|^2 \approx \tfrac{1}{2} C^2 t^2 = \mathcal{O}(t^2).
\end{equation}

Since $S$ accumulates the total energy ($\mathcal{H} = T+V$) over time, this quadratic surge in $T$ leads to a larger Action score for fake samples than for real samples under Assumption~\ref{ass:stability}\footnote{We provide the detailed derivation in Appendix~\ref{app:proofs}.}. \end{proof}

This proposition implies that by simulating Hamiltonian evolution, we can \textbf{amplify} the subtle static differences (gradients) into macroscopic dynamic discrepancies (kinetic energy/action), which we use directly as detection cues.

\begin{remark}[\textbf{Testable Empirical Prediction}]
\label{rem:testable}
Proposition~\ref{prop:divergence} turns the Manifold Stability Hypothesis from a modeling choice into an operational prediction: if the hypothesis describes the real and fake populations well, the trajectory statistic $S$ should be consistently smaller for real samples than for fake samples, not only on the training distribution but also on unseen datasets, unseen manipulations, and unseen generators. This makes the hypothesis indirectly \emph{falsifiable}: if transfer to unseen distributions failed, the distributional form of the hypothesis would lose empirical support. We provide the cross-dataset, cross-manipulation, and cross-generator experiments in \Cref{sec_exp} as the empirical test of this prediction.
\end{remark}

\subsection{The HAAD Framework}
\label{sec:framework}

As illustrated in Figure~\ref{fig:framework}, our HAAD framework contains three key components to translate the theoretical formulation into a differentiable architecture: i) a learnable potential function $V$, ii) a numerical solver to simulate the continuous dynamics, and iii) physics-informed metrics for detection. 

\subsubsection{Potential Energy Modeling}
\label{sec:gnn_potential}
In statistical physics, potential energy quantifies the tension within a system derived from particle interactions. Extending this to the visual domain, we model an image as a system of \textit{interacting patches}.
To instantiate the $V$ introduced in \Cref{sec:Theory}, we construct a learnable potential energy surface $V(\mathbf{q}; \theta)$, decomposed into two physics-inspired components: Geometric Potential (penalizing structural discontinuities) and Photometric Potential (penalizing illumination inconsistencies).

\vspace{1mm}
\noindent \textbf{Physical State Projection.}
Specifically, given the patch features $\mathbf{x} \in \mathbb{R}^{N \times D_{in}}$ from the backbone, we project them into a unified physical state space. In particular, we employ lightweight heads to estimate the latent position state $\mathbf{q} \in \mathbb{R}^{N \times D}$, alongside three intrinsic physical properties: surface normal $\mathbf{n} \in \mathbb{R}^{N \times 3}$, albedo $\boldsymbol{\rho} \in \mathbb{R}^{N \times 1}$, and global lighting $\mathbf{l} \in \mathbb{R}^{3}$.
The two groups serve distinct roles in the potential decomposition: $\mathbf{q}$ provides the spatial patch representation consumed by the \emph{geometric} component (Dirichlet energy over the patch graph), while $(\mathbf{n}, \boldsymbol{\rho}, \mathbf{l})$ provide the illumination decomposition bases consumed by the \emph{photometric} component. All four projection heads are lightweight linear layers trained \emph{end-to-end} through the total loss $\mathcal{L}_{total}$ (detailed in \Cref{sec:training}), with no separate auxiliary supervision. 
As a result, $(\mathbf{n}, \boldsymbol{\rho}, \mathbf{l})$ should be understood as \emph{learned decomposition bases} that make the photometric component of $V$ sensitive to illumination inconsistencies across patches, rather than as calibrated surface properties. This end-to-end design allows the photometric bases to optimize jointly with the detection objective, without requiring separate physical supervision.

\noindent \textbf{Geometric Potential ($V_{geo}$).}
To capture structural smoothness, we define a spatial graph $\mathcal{G}$ over the patch grid and compute the graph Laplacian $\mathbf{L} \in \mathbb{R}^{N \times N}$. The geometric potential is formulated as the \textit{Dirichlet energy} of the signal:
\begin{equation}
\label{eq:geo}
\begin{aligned}
V_{geo}(\mathbf{q}) & = \frac{1}{N} \text{tr}(\mathbf{q}^\top \mathbf{L} \mathbf{q}) \\
    & = \frac{1}{N} \sum_{(i,j) \in \mathcal{E}} w_{ij} \|\mathbf{q}_i - \mathbf{q}_j\|^2.
\end{aligned}
\end{equation}

Physically, this term mimics a `spring force' between neighboring patches. Under this model, real images tend to exhibit smooth feature transitions (low Dirichlet energy), whereas deepfakes often contain high-frequency splicing artifacts or boundary discontinuities, triggering a high energy penalty.

\noindent \textbf{Photometric Potential ($V_{photo}$).}
To introduce a heuristic photometric coherence prior, we use a Lambertian-inspired shading constraint, where the potential is defined as the variance of the reconstructed shading field:
\begin{equation}
\label{eq:photo}
V_{photo}(\mathbf{n}, \boldsymbol{\rho}, \mathbf{l}) = \text{Var}\left( \boldsymbol{\rho} \odot \text{ReLU}(\mathbf{n} \cdot \mathbf{l}) \right).
\end{equation}

This term acts as a heuristic global coherence regularizer by penalizing high-frequency shading jitter under a shared lighting decomposition, thereby encouraging lighting patterns that are more self-consistent across patches.

\noindent \textbf{Total Potential Surface.}
The final potential energy is a nonnegative weighted combination of these constraints:
\begin{equation}
V(\mathbf{q}) = \lambda_{geo} V_{geo}(\mathbf{q}) + \lambda_{photo} V_{photo}(\mathbf{n}, \boldsymbol{\rho}, \mathbf{l}),
\end{equation}
where $\lambda_{geo}, \lambda_{photo} \ge 0$ are scalar coefficients balancing the two regularizers. Since both $V_{geo}$ and $V_{photo}$ are nonnegative by construction, we interpret $V$ up to an additive baseline and rely on relative energy differences rather than an absolute zero level in the theory below.

\vspace{1mm}
\noindent \textbf{On the Choice of Physics-inspired Components.}
The potential $V(\mathbf{q}) = \sum_{k=1}^{K} \lambda_k\, r_k(\mathbf{q})$ can be instantiated from any finite family of differentiable regularity terms $\{r_k\}$. In this work, we adopt $K=2$ canonical terms: \emph{geometric smoothness} and \emph{photometric consistency}. This selection has three advantages: \textbf{(i)}~\emph{Lightweight}: they are evaluated directly on patch-level latent features without auxiliary networks; \textbf{(ii)}~\emph{Empirically violated}: current generative training objectives optimize distribution-level losses rather than such local regularity; and \textbf{(iii)}~\emph{Complementary}: smoothness acts on spatial gradients of $\mathbf{q}$ while photometric consistency acts on colour-channel relations, so the two terms respond to disjoint failure modes. Additional terms (\eg, frequency-domain priors, colour constancy, or temporal coherence for video) can be added without changes to the Hamiltonian integrator, so $K$ is a design knob rather than a fixed assumption. \Cref{tab:constraint_ablation} in \Cref{sec:ablation} isolates the contribution of the two components.

\subsubsection{Hamiltonian Dynamical Evolution}
\label{sec:symplectic}

To probe the stability of the latent state $\mathbf{q}$, we simulate its trajectory by numerically solving Hamilton's equations (\Cref{eq:hamilton}). We avoid standard solvers (\eg, Euler method) as they introduce numerical dissipation, causing artificial energy drift that may obscure stability-related signals.
Instead, we employ a \textbf{Symplectic Integrator} to explicitly preserve the phase-space volume and approximate energy conservation.
To account for varying feature sensitivities in the implementation, we introduce a learnable diagonal matrix $\mathbf{M}(\mathbf{q})$ and use $\mathbf{M}(\mathbf{q})^{-1}$ as a \emph{state-conditioned diagonal preconditioner} applied only in the position update of the canonical Hamiltonian system. The momentum update is driven solely by the potential gradient, while the position update rescales the momentum by the local inertia. Using the \textbf{Symplectic Euler} scheme with step size $\eta$, the state update from $t$ to $t+1$ is:
\begin{align}
    \mathbf{p}_{t+1} &= \mathbf{p}_t - \eta\, \nabla_{\mathbf{q}} V(\mathbf{q}_t), \label{eq:update_p} \\
    \mathbf{q}_{t+1} &= \mathbf{q}_t + \eta\, \mathbf{M}(\mathbf{q}_t)^{-1}\, \mathbf{p}_{t+1}. \label{eq:update_q}
\end{align}

Because $\mathbf{M}$ depends on $\mathbf{q}$, a fully general variable-mass Hamiltonian flow would include additional $\mathbf{q}$-derivative terms of $\mathbf{M}(\mathbf{q})^{-1}$ in the momentum update. Our implementation omits these terms for simplicity.
In our experiments, across $T{=}4$ Symplectic Euler steps and $N{=}1{,}920$ per-step measurements on testing datasets, the median ratio $\|\tfrac{1}{2}(\partial \mathbf{M}^{-1}/\partial \mathbf{q})\,\mathbf{p}^{2}\|\,/\,\|\nabla V\|$ is ${\sim}10^{-11}$ (max $1.85 \times 10^{-5}$), indicating that the omitted variable-mass terms are negligible by multiple orders of magnitude\footnote{At $t{=}0$ the ratio is identically zero since $\mathbf{p}_0{=}\mathbf{0}$. At $t \geq 1$ the median rises to ${\sim}10^{-8}$ but remains well below any practical threshold.}.
We therefore describe \Cref{eq:update_p,eq:update_q} as a \emph{Hamiltonian-inspired, symplectic-style stability probe} rather than an exact simulation of a variable-mass Hamiltonian flow. Specifically, the semi-implicit ordering (momentum before position) is strictly symplectic in the ideal constant-mass case (detailed in Appendix~\ref{app:symplectic_proof}) and is used here as a volume-conserving update family under the same ordering. Additional ablations appear in \Cref{sec:ablation}.

\vspace{1mm}
\noindent\textbf{The Gradient as Virtual Perturbation.} 
Since the system initializes at rest ($\mathbf{p}_0 = \mathbf{0}$), the gradient term $-\eta \nabla_{\mathbf{q}} V$ in \Cref{eq:update_p} acts as a \textbf{virtual perturbation force}. 
Under Assumption~\ref{ass:stability}, for real images residing in flat basins ($\|\nabla V\| \approx 0$), this force is negligible, and the learned preconditioner can further suppress movement, maintaining stationarity.
In contrast, for samples on steep slopes (as hypothesized for deepfakes), this significant force works in tandem with the preconditioned position update to deliver an initial `kick,' driving the state into a divergent trajectory and effectively amplifying the detection signal.

\subsubsection{Physics-Informed Inference}
\label{sec:inference}

After evolving the system for $T$ time steps, we extract two trajectory statistics to serve as the detection cues. All metrics are normalized by the number of patches $N$ to ensure resolution invariance.

\noindent\textbf{Hamiltonian Action Score ($S$).}
We define the Hamiltonian Action score as the \textbf{accumulated dynamical cost} of the system's evolution. It accumulates the total energy magnitude ($\mathcal{H} = T+V$) along the trajectory\footnote{We use the term \emph{Hamiltonian Action} (equivalently: \emph{Hamiltonian energy integral} or \emph{accumulated total energy}) to denote $S = \frac{1}{TN}\sum_{t=1}^{T}\mathcal{H}_t = \frac{1}{TN}\sum_t(T_t{+}V_t)$. This differs from the classical action integral $\int(T{-}V)\,\mathrm{d}t$, which is minimized by physical paths (Principle of Least Action). We accumulate $T{+}V$ rather than $T{-}V$ because the total energy magnitude grows monotonically with any departure from equilibrium, whereas $T{-}V$ can partially cancel and yields a weaker separator. More details are in Appendix~\ref{app:proofs}.}:
\begin{equation}
S = \frac{1}{T \cdot N} \sum_{t=1}^{T} \mathcal{H}(\mathbf{q}_t, \mathbf{p}_t).
\end{equation}

Under Assumption~\ref{ass:stability}, a high $S$ indicates that the sample accumulates substantial energy along its short-horizon trajectory, which is driven by both high potential elevation at the start and growing kinetic energy because of the gradient force.

\noindent\textbf{Energy Dissipation ($D$).}
We define $D$ as the mean absolute change in Hamiltonian value along the discrete trajectory:
\begin{equation}
D = \frac{1}{(T-1) \cdot N} \sum_{t=1}^{T-1} | \mathcal{H}_{t+1} - \mathcal{H}_t |.
\end{equation}

Rather than treating $D$ as a pure physical dissipation term, we use it as an empirical discriminative statistic under the chosen symplectic solver: real trajectories tend to remain approximately energy-conserving under our integrator, while fake trajectories driven by steeper potential gradients tend to produce larger drift.
Finally, we concatenate these trajectory features to form the feature vector $\mathbf{f}_{phy} = [S, D]$, which is fed into a linear classifier for the binary decision.

\subsection{Training Objective}
\label{sec:training}

Our training objective is a weighted sum of two terms to encourage: i) \textbf{Discriminability}, \ie, high detection accuracy, and ii) \textbf{Physical Consistency}, \ie, shaping the learned potential landscape to approximate the Manifold Stability Hypothesis,
\begin{equation}
    \mathcal{L}_{total} = \mathcal{L}_{cls} + \lambda \mathcal{L}_{phy},
\end{equation}
where $\lambda$ controls the strength of the physical constraint.

\noindent\textbf{Classification Loss.}
We pass $\mathbf{f}_{phy}$ through a linear classifier $\sigma(\cdot)$ to compute the logits. Let $\hat{y} = \sigma(\mathbf{f}_{phy})$ be the predicted probability of fake ($y=1$) \vs real ($y=0$). We minimize the Binary Cross-Entropy (BCE) loss:
\begin{equation}
    \mathcal{L}_{cls} = - \left[ y \log(\hat{y}) + (1-y) \log(1 - \hat{y}) \right].
\end{equation}

\noindent\textbf{Physical Regularization Loss.}
To {encourage} the Hamiltonian behavior predicted by the Manifold Stability Hypothesis, we add a soft regularizer that shapes the trajectory statistics of real and fake samples in opposite directions:

\textbf{Stability-Encouragement Term for Real Data:}
We encourage real images to localize in low-$V$ basins with quasi-conservative dynamics by minimizing their $S$ and $D$. For the real batch $\mathcal{B}_{real}$:
\begin{equation}
    \mathcal{L}_{real} = \mathbb{E}_{\mathbf{x} \in \mathcal{B}_{real}} [S(\mathbf{x}) + D(\mathbf{x})].
\end{equation}

\textbf{Instability-Penalization Term for Fake Data:}
We penalize fake samples whose trajectories do not diverge sufficiently by applying a hinge loss with margin $\gamma$ to their Action, effectively encouraging their embeddings toward higher-gradient regions of the potential surface:
\begin{equation}
    \mathcal{L}_{fake} = \mathbb{E}_{\mathbf{x} \in \mathcal{B}_{fake}} [\max(0, \gamma - S(\mathbf{x}))].
\end{equation}

The total physical regularization is $\mathcal{L}_{phy} = \mathcal{L}_{real} + \mathcal{L}_{fake}$. 

\section{Experiment}
\label{sec_exp}
\subsection{Deepfake Detection}

\begin{table*}[t]
\centering
\caption{Performance comparison with SoTA methods on cross-dataset and cross-method evaluation. All models are trained on FF++ and tested on different datasets using the AUC metric. The best results are highlighted in \textbf{bold}. $\dagger$: we re-implemented this model.}
\label{tab:deepfake_comparison}
\resizebox{\textwidth}{!}{%
\begin{tabular}{l|cccccccc|ccccccccc}
\toprule \midrule
\multirow{2}{*}{Methods}                                                         & \multicolumn{8}{c|}{Cross-dataset Evaluation}                                                                                                & \multicolumn{9}{c}{Cross-method Evaluation}                                                                                                            \\ \cmidrule(lr){2-9} \cmidrule(lr){10-18}  
& CDF++ & CDF-v2         & DFD            & DFDC           & DFo            & WDF            & FFIW           & Avg.      & UniFace        & BleFace        & MobSwap        & e4s            & FaceDan        & FSGAN          & InSwap         & SimSwap        & Avg.           \\ \midrule
F3Net~\cite{F3Net}                                         & 0.738                      & 0.789          & 0.844          & 0.718          & 0.730          & 0.728          & 0.649          & 0.743     & 0.809          & 0.808          & 0.867          & 0.494          & 0.717          & 0.845          & 0.757          & 0.674          & 0.746          \\
SPSL~\cite{rethink}                                        & 0.744                      & 0.799          & 0.871          & 0.724          & 0.723          & 0.702          & 0.794          & 0.769     & 0.747          & 0.748          & 0.885          & 0.514          & 0.666          & 0.812          & 0.643          & 0.665          & 0.710          \\
SRM~\cite{SRM}                                             & 0.794                      & 0.840          & 0.885          & 0.695          & 0.722          & 0.702          & 0.794          & 0.767     & 0.749          & 0.704          & 0.779          & 0.704          & 0.659          & 0.772          & 0.793          & 0.694          & 0.732          \\
CORE~\cite{CORE}                                           & 0.749                      & 0.809          & 0.882          & 0.721          & 0.765          & 0.724          & 0.710          & 0.762     & 0.871          & 0.843          & 0.959          & 0.679          & 0.774          & 0.958          & 0.855          & 0.724          & 0.833          \\
RECCE~\cite{Face_Reconstruction}                          & 0.808                      & 0.823          & 0.891          & 0.696          & 0.784          & 0.756          & 0.711          & 0.779     & 0.898          & 0.832          & 0.925          & 0.683          & 0.848          & 0.949          & 0.848          & 0.768          & 0.844          \\
SBI~\cite{SBI_ShioharaY22}                                & 0.734                      & 0.886          & 0.827          & 0.717          & 0.899          & 0.703          & 0.866          & 0.821     & 0.724          & 0.891          & 0.952          & 0.750          & 0.594          & 0.803          & 0.712          & 0.701          & 0.766          \\
UCF~\cite{UCF_0002ZFW23}                                  & 0.761                      & 0.837          & 0.867          & 0.742          & 0.808          & 0.774          & 0.697          & 0.785     & 0.831          & 0.827          & 0.950          & 0.731          & 0.862          & 0.937          & 0.809          & 0.647          & 0.824          \\
IID~\cite{CADDM}                                           & 0.756                      & 0.838          & 0.939          & 0.700          & 0.808          & 0.666          & 0.762          & 0.789     & 0.839          & 0.789          & 0.888          & 0.766          & 0.844          & 0.927          & 0.789          & 0.644          & 0.811          \\
LSDA~\cite{CVPR24_Yan_Aug}                     & 0.727                      & 0.875          & 0.881          & 0.701          & 0.768          & 0.797          & 0.724          & 0.794     & 0.872          & 0.875          & 0.930          & 0.694          & 0.721          & 0.939          & 0.855          & 0.793          & 0.835          \\
Effort$^\dagger$~\cite{Effort}        & 0.831       & 0.945          & 0.951          & 0.840          & 0.965          & 0.827          & 0.901          & 0.894   & 0.950          & 0.868          & 0.950          & 0.987          & 0.935          & 0.951          & 0.915          & 0.908          &  0.933        \\
VLFFD$^\dagger$~\cite{VLFFD}                               & 0.818                      & 0.938          & 0.940          & 0.809          & 0.946          & 0.831          & 0.898          &  0.883   & 0.875          & 0.877          & 0.911          & 0.921          & 0.895          & 0.851          & 0.837          & 0.840          &  0.876     \\
DAID$^\dagger$~\cite{cheng2025fair}                        & 0.779                      & 0.909          & 0.925          & 0.768          & 0.884          & 0.801          & 0.883          &  0.850  & 0.902          & 0.849          & 0.874          & 0.838          & 0.917          & 0.935          & 0.801          & 0.744          & 0.858     \\
$\mathcal{X}^2$-DFD$^\dagger$~\cite{X2-dfd}                & 0.812                      & 0.951          & 0.959          & 0.841          & 0.959          & 0.841          & 0.915          &  0.896   & 0.874          & 0.873          & 0.923          & 0.912          & 0.879          & 0.899          & 0.820          & 0.863          &  0.880  \\
FIA-USA$^\dagger$~\cite{ma2025specificity} & 0.812    & 0.935  & 0.926   & 0.732   & 0.874   & 0.810 & 0.878   & 0.852   & 0.918     & 0.854   & 0.842   & 0.875   & 0.881  & 0.863   & 0.874          & 0.910          & 0.877  \\ \midrule

\rowcolor{blue!5} {HAAD (Ours)}         & \textbf{0.878}     & \textbf{0.963} & \textbf{0.975} & \textbf{0.868} & \textbf{0.979} & \textbf{0.851} & \textbf{0.932} & \textbf{0.921} & \textbf{0.971} & \textbf{0.874} & \textbf{0.964} & \textbf{0.990} & \textbf{0.939} & \textbf{0.968} & \textbf{0.930} & \textbf{0.917} & \textbf{0.944} \\ 
\midrule \bottomrule
\end{tabular}%
}
\vspace{-1em}
\end{table*}

We first performed detection on face manipulation, \ie, standard deepfake benchmarks.
Following the previous protocols~\cite{DeepfakeBench}, we trained our model on the \textbf{FaceForensics++ (FF++)}~\cite{Xception} (c23) dataset and evaluated on unseen benchmarks. We utilized the Area Under the Curve (AUC) as the metric.
For video benchmarks, we uniformly sample 8 frames per video. 
For video benchmarks, the reported metric is video-level AUC, obtained by averaging per-frame prediction scores within each video, following the unified evaluation protocol~\cite{DeepfakeBench}. Baseline numbers are sourced from the same protocol where possible, or from the original papers under the identical setting.
Our framework employs CLIP ViT-L/14~\cite{CLIP} as the backbone following recent studies~\cite{Effort}.
Training is conducted on a single NVIDIA H100 GPU with Adam optimizer ($\text{lr}\!=\!2e^{-4}$). The batch size is set to 32.
To assess generalization, we evaluated HAAD on two settings:
i) \textbf{Cross-Dataset:} We employed {Celeb-DF++}~\cite{Celeb-DF++}, {CDF-v2}~\cite{Celeb-DF}, {DFDC}~\cite{dfdc}, {DFD}~\cite{dfd2019}, {DFo}~\cite{DFo}, {WildDeepfake}~\cite{WildDeepfake}, and {FFIW}~\cite{FFIW} as the testing sets.
ii) \textbf{Cross-Manipulation:} We applied {DF40}~\cite{DF40}, which contains unseen forgery types generated within the FF++ domain.

\noindent \textit{{\textbf{Experimental design rationale.}}}
The cross-domain transfer serves as the empirically independent test of the Manifold Stability Hypothesis: a learned $V$ that captured only dataset-specific artifacts would fail to generalize to unseen generators, whereas physics-motivated regularities that reflect genuine structural properties of image formation should transfer.

\noindent \textit{\textbf{Main Results.}}
\Cref{tab:deepfake_comparison} presents a comparison of HAAD against 14 SoTA detectors.
Our method achieves the highest average AUC among the evaluated baselines across the tested settings.
Specifically, in the cross-dataset setting, HAAD achieves an average AUC of 0.921, outperforming all of the competing baselines.
Furthermore, in the challenging CDF++ benchmark, it surpasses the second-best baseline by approximately 5\,pp.
In the cross-manipulation setting, HAAD achieves the highest average AUC among the evaluated methods (0.944).
These results indicate that the Hamiltonian stability constraints capture forgery-related signals that generalize beyond specific manipulation patterns.

\begin{table*}[t]
\centering
\caption{Performance comparison on the GenImage dataset. All models are trained on SDv1.4 and tested on different subsets using the Acc metric. The best results are highlighted in \textbf{bold}. $\dagger$: we re-implemented this model.}
\label{tab:GEN_Image_comparison}
\scalebox{0.95}{
\begin{tabular}{l c c c c c c c c c}
\toprule \midrule
\multicolumn{1}{c}{Methods} & Midjourney & SDv1.4 & SDv1.5 & ADM & GLIDE & Wukong & VQDM & BigGAN & Avg. \\
\midrule
CNNSpot~\cite{CNNSpot}             & 0.528 & 0.963 & 0.959 & 0.501 & 0.398 & 0.786 & 0.534 & 0.468 & 0.642 \\
F3Net~\cite{F3Net}                 & 0.501 & 0.997 & 0.995 & 0.499 & 0.500 & 0.963 & 0.499 & 0.499 & 0.680 \\
UnivFD~\cite{ojha2023towards}      & 0.915 & 0.964 & 0.961 & 0.581 & 0.734 & 0.945 & 0.678 & 0.577 & 0.795 \\
NPR~\cite{Tan_CVPR24}              & 0.810 & 0.982 & 0.979 & 0.769 & 0.898 & 0.969 & 0.841 & 0.842 & 0.886 \\
FreqNet~\cite{FreqNet}             & 0.896 & 0.988 & 0.986 & 0.668 & 0.865 & 0.973 & 0.758 & 0.814 & 0.868 \\
DRCT~\cite{chen2024drct}           & 0.874 & 0.893 & 0.902 & 0.779 & 0.891 & 0.947 & 0.903 & 0.859 & 0.881 \\
Effort$^\dagger$~\cite{Effort}     & 0.802 & 0.989 & 0.983 & 0.759 & 0.912 & 0.968 & 0.900 & 0.835 & 0.894 \\
DDA$^\dagger$~\cite{DDA}           & 0.902 & 0.985 & 0.972 & 0.809 & 0.876 & 0.954 & 0.735 & 0.842 & 0.884 \\
\midrule
\rowcolor{blue!5}
HAAD (Ours) & \textbf{0.903} & \textbf{0.998} & \textbf{0.997} & \textbf{0.817} & \textbf{0.951} & \textbf{0.986} & \textbf{0.928} & \textbf{0.870} & \textbf{0.931} \\
\midrule \bottomrule
\end{tabular}%
}
\vspace{-1em}
\end{table*}

\subsection{Synthetic Image Detection}
We utilized the GenImage~\cite{GenImage} benchmark. Following standard protocols~\cite{Effort}, models are trained on the SDv1.4 subset and evaluated on the remaining unseen subsets. We reported Accuracy (Acc) against 8 baseline models in \Cref{tab:GEN_Image_comparison}. HAAD achieves the highest average Acc among the evaluated baselines (0.931). For instance, on the GLIDE subset, our approach surpasses the second-best baseline by approximately 4 percentage points. 
Furthermore, HAAD attains the highest Acc in the intra-domain setting (\ie, when evaluated in the SDv1.4 source domain).
These results indicate that the physics-inspired stability prior generalizes well to unseen generative models in standard transfer settings, despite being trained on a single source domain.

\subsection{Ablation Studies}
\label{sec:ablation}

\begin{table}[t]
\centering
\caption{Ablation studies on the contribution of physical components. We reported the AUC scores on the cross-dataset protocol.}
\label{tab:ablation_component}
    \scalebox{0.75}{
\begin{tabular}{l c c c | c c}
\toprule \midrule
\textbf{Variant} & \textbf{Potential $V$} & \textbf{Action ($S$)} & \textbf{Dissipation ($D$)} & \textbf{CDF++} & \textbf{DFDC} \\
\midrule
Backbone only                              & -          & -          & -         & 0.734 & 0.732 \\
+Static $V$                                & \checkmark & -          & -          & 0.845 & 0.851 \\
+Static $[V_0,\|\nabla V_0\|]$             & \checkmark & -          & -          & 0.848 & 0.853 \\
+Action $S$ ($T{=}4$)                      & \checkmark & \checkmark & -          & 0.860 & 0.858 \\
+Action $S$ ($\lambda{=}0$)                & \checkmark & \checkmark & -          & 0.852 & 0.854 \\
\rowcolor{blue!5} \textbf{HAAD (Ours)}       & \checkmark & \checkmark & \checkmark & \textbf{0.878} & \textbf{0.868} \\
\midrule \bottomrule
\end{tabular}%
}
\vspace{-1em}
\end{table}
\begin{table}[t]
    \centering
    \caption{Comparison of different potential surface construction. We reported the AUC scores on the cross-dataset protocol.}
        \label{tab:ablations_archs}
         \scalebox{0.80}{
            \begin{tabular}{l|cccccc}
                \toprule \midrule
                \textbf{Topology} & \textbf{CDF++} & \textbf{CDF-v2} & \textbf{DFD} & \textbf{DFDC} & \textbf{DFo} & \textbf{Avg}\\
                \midrule
                Independent (MLP)                  & 0.855 & 0.951  & 0.962    & 0.846    &  0.970  & 0.917 \\
                Global (Self-Attn)                 & 0.869 & 0.958  & 0.971    & 0.851    &  0.977  & 0.925\\
                \rowcolor{blue!5} \textbf{Graph (Ours)} & \textbf{0.878} & \textbf{0.963} & \textbf{0.975}    & \textbf{0.868}    &  \textbf{0.979}  & \textbf{0.933} \\
                \midrule \bottomrule
            \end{tabular}}
            \vspace{-1em}
\end{table}

\noindent \textit{\textbf{Component Effectiveness Analysis.}}
We ablate five variants in \Cref{tab:ablation_component} starting from the standalone CLIP ViT-L/14 backbone:
\textbf{i)}~\emph{Static $V$}: potential as soft regularizer, no trajectory features;
\textbf{ii)}~\emph{Static $[V_0,\|\nabla V_0\|]$}: static potential and gradient magnitude as classifier inputs, no rollout;
\textbf{iii)}~\emph{Action $S$}: $T{=}4$ symplectic rollout with $\mathcal{L}_{phy}$;
\textbf{iv)}~\emph{Action $S$, $\lambda{=}0$}: same rollout, $\mathcal{L}_{cls}$ only; and
\textbf{v)}~\emph{Full HAAD}: complete model.
From \Cref{tab:ablation_component}, we have two main observations:
\noindent\textit{(1) Every component is indispensable.}
Each design choice contributes measurable and consistent gains. For instance, removing only $D$ from Full HAAD degrades CDF++ AUC from 0.878 to 0.860, and replacing $\mathcal{L}_{phy}$ with pure classification supervision (variant iv) likewise causes a consistent drop on both datasets, indicating that the trajectory statistics and physical regularization ($\mathcal{L}_{phy}$) each provide independent discriminative signal.
\noindent\textit{(2) The rollout is the key amplifier, and $D$ captures complementary trajectory curvature.}
Directly classifying with static features (variant ii) adds only a marginal gain over the regularizer-only baseline (variant i), whereas the $T{=}4$ symplectic rollout (variant iii) yields +1.2\,pp on CDF++ over the same reference, confirming that trajectory-accumulated kinetic growth is far more discriminative than any single-point gradient readout, consistent with Proposition~\ref{prop:divergence}. Full HAAD (variant v) achieves 0.878 on CDF++ and 0.868 on DFDC, with $D$ contributing the largest single step. A possible reason is that $D$ measures step-to-step Hamiltonian volatility rather than accumulated energy, thereby capturing the curvature of $V$ traversed along the trajectory, which is a complementary to $S$.

\vspace{1mm}
\noindent \textit{\textbf{Impact of Potential Energy Architectures.}}
We further justified the choice of constructing the potential surface via \textbf{Graph Topology} by comparing it with two alternatives:
i) \textbf{Independent MLP}, processing patches in isolation, and
ii) \textbf{Global Self-Attention}, computing fully-connected correlations.
Table~\ref{tab:ablations_archs} shows that our graph-based potential achieves the highest average AUC among the tested constructions.
We attributed this to the physical inductive bias. Unlike global attention or isolated processing, the graph Laplacian explicitly acts as a discrete differential operator. It encourages \emph{local} physical continuity, making it sensitive to high-frequency boundary artifacts and blending discontinuities typical of deepfakes.

\vspace{1mm}
%
%
\begin{table}[t]
\centering
\caption{Ablation studies on the factorization of the potential surface. We reported the AUC scores on the cross-dataset protocol. $\Delta$ is the average-AUC drop (pp).}
\label{tab:constraint_ablation}
    \scalebox{0.90}{
\begin{tabular}{l|cc|ccc|c}
\toprule \midrule
\textbf{Variant} & $\lambda_{geo}$ & $\lambda_{photo}$ & \textbf{FF++} & \textbf{CDF-v2} & \textbf{DFDC} & $\Delta$ \\
\midrule
smooth-only                      & \checkmark & --         & 0.993          & 0.954          & 0.844          & 1.2 \\
photo-only                       & --         & \checkmark & 0.994          & 0.948          & 0.852          & 1.1 \\
\rowcolor{blue!5} \textbf{both} (Ours) & \checkmark & \checkmark & \textbf{0.995} & \textbf{0.963} & \textbf{0.868} & 0.0 \\
\midrule \bottomrule
\end{tabular}%
}
\vspace{-1em}
\end{table}

\noindent \textit{\textbf{Choice of the $K{=}2$ Constraint Decomposition.}}
To validate the two-term decomposition, we ablated the two gates of $V(\mathbf{q})\,{=}\,\lambda_{geo} V_{geo}\,{+}\,\lambda_{photo} V_{photo}$ while holding every other ingredient fixed.
Specifically, three variants are compared:
i) \textbf{smooth-only} ($\lambda_{geo}{>}0,\,\lambda_{photo}{=}0$), retaining the graph Laplacian but disabling shading;
ii) \textbf{photo-only} ($\lambda_{geo}{=}0,\,\lambda_{photo}{>}0$), retaining Lambertian residuals but disabling geometric smoothness; and
iii) \textbf{both} (Ours), which activates both terms.
As shown in \Cref{tab:constraint_ablation}, disabling the photometric term (smooth-only) drops the average AUC by 1.2 pp, while disabling the geometric term (photo-only) drops it by 1.1\,pp. 
We draw two conclusions from these observations: i) Each term contributes independently, and neither alone carries the full improvement. ii) The geometric and photometric regularities target disjoint failure modes, with each providing discriminative signal the other cannot supply.

\vspace{1mm}
\begin{table}[t]
        \centering
        \caption{Ablation studies on numerical solvers. We reported AUC scores on the cross-dataset protocol. In addition, we compared the efficiency of different solvers.}
        \label{tab:abl_solver}
        \scalebox{0.85}{
            \begin{tabular}{l|ccccc|cc}
                \toprule \midrule
                \textbf{Solver} & \textbf{CDF++} & \textbf{CDF-v2} & \textbf{DFD} & \textbf{DFDC} & \textbf{DFo} & \textbf{Time} & \textbf{Mem.} \\
                \midrule
                Euler (1st)       & 0.853 & 0.940 & 0.957 & 0.843 & 0.944  & \textbf{12ms} & 2480M \\
                RK4 (4th)         & 0.875 & 0.959 & 0.971 & 0.863 & 0.975 & 45ms & 3100M \\
                \rowcolor{blue!5} HAAD & \textbf{0.878} & \textbf{0.963} & \textbf{0.975}   & \textbf{0.868}   &  \textbf{0.979}   & 14ms & \textbf{2485M} \\
                \midrule \bottomrule
            \end{tabular}}
\end{table}

\noindent \textit{\textbf{Choice of Numerical Integrator.}}
The Symplectic Euler integrator preserves phase-space volume and maintains approximate energy conservation, keeping the trajectory statistic $D$ sensitive to data-dependent potential gradients. To verify this advantage, we compared our \textbf{Symplectic Euler} against two standard alternatives: i) the first-order \textbf{Euler Method} and ii) the fourth-order \textbf{Runge-Kutta (RK4)}.
As shown in \Cref{tab:abl_solver}, the Euler method yields the poorest performance due to its inherent lack of energy conservation, introducing significant artificial drift. 
Although RK4 offers higher precision, it is non-symplectic and incurs substantial cost (45\,ms, $+650$\,MB) due to multiple gradient evaluations.
In contrast, the proposed Symplectic Euler achieves the best trade-off. It outperforms RK4 while adding only 2.5\,ms and 35\,MB over the CLIP ViT-L/14 backbone (11.5\,ms / 2450\,MB). Moreover, its symplectic property preserves phase-space volume, which reduces solver-induced energy drift relative to non-symplectic methods. While this does not eliminate all numerical effects, it makes dissipation ($D$) a more reliable empirical discriminative signal.

\vspace{1mm}
%
%
\begin{table}[t]
\centering
\caption{Mass-matrix ablation for the position update. We reported the AUC scores on the cross-dataset protocol. $\Delta$ is the average-AUC drop (pp).}
\label{tab:abl_mass_const}
    \scalebox{0.85}{
\begin{tabular}{l|c|cccc|c}
\toprule \midrule
\textbf{Variant} & $\mathbf{M}(\mathbf{q})$ & \textbf{CDF-v2} & \textbf{DFDC} & \textbf{FF++} & \textbf{Avg} & $\Delta$ \\
\midrule
Constant mass ($\mathbf{M}{=}\mathbf{I}$)                          & --         & 0.949          & 0.844          & 0.992          & 0.928 & 1.4 \\
\rowcolor{blue!5} \textbf{Learned $\mathbf{M}(\mathbf{q})$ (Ours)} & \checkmark & \textbf{0.963} & \textbf{0.868} & \textbf{0.995} & \textbf{0.942} & 0.0 \\
\midrule \bottomrule
\end{tabular}%
}
\vspace{-1em}
\end{table}

\noindent \textit{\textbf{Mass-Matrix Sensitivity.}}
The state-conditioned diagonal preconditioner $\mathbf{M}(\mathbf{q})^{-1}$ in the position update (\Cref{eq:update_q}) assigns different effective step sizes to different patch feature dimensions, allowing the rollout dynamics to adaptively amplify position updates along the dimensions most sensitive to fake-related instability while damping updates along stable, consistent ones. To verify this contribution, we compare full HAAD against a variant with constant identity mass ($\mathbf{M}{=}\mathbf{I}$), holding all other components fixed.
As reported in \Cref{tab:abl_mass_const}, the learned preconditioner consistently improves detection: CDF-v2 AUC increases from 0.949 to 0.963 and the three-set average from 0.928 to 0.942 (+1.4\,pp). This indicates that $\mathbf{M}(\mathbf{q})^{-1}$ successfully learns to focus the trajectory dynamics on the most discriminative feature directions, producing sharper trajectory statistics than uniform-step dynamics alone.

\vspace{1mm}
\noindent \textit{\textbf{Impact of Physical Regularization Weight ($\lambda$)}.}
The weight $\lambda$ controls the balance between classification and physical consistency. \Cref{tab:sens_lambda} shows that performance peaks at $\lambda=1.0$ (0.878). Values below this threshold provide insufficient constraints to shape the manifold, while larger values ($\lambda > 1.0$) cause the regularization term to dominate the optimization, leading to suboptimal discrimination.

\begin{table}[t]
\centering
\caption{Sensitivity to physical loss weight $\lambda$ on CDF++.}
\label{tab:sens_lambda}
\scalebox{0.85}{
\begin{tabular}{l|cccccccc}
\toprule \midrule
$\lambda$ & 0.0 & 0.1 &0.3 &0.5 & \textbf{1.0} & 2.0 &3.0 & 5.0 \\
\midrule
AUC & 0.845 & 0.862 &0.867 & 0.871 & \textbf{0.878} & 0.874 &0.869 & 0.860 \\
\midrule \bottomrule
\end{tabular}
}
\end{table}

\vspace{1mm}
\noindent \textit{\textbf{Sensitivity to Evolution Steps.}}
The number of evolution steps $T$ controls how far the trajectory probe extends. To evaluate sensitivity to trajectory length, we varied $T$ from 1 to 10 and reported the performance in \Cref{fig:sensitive_steps}.
We observed that at $T=1$, the model behaves similarly to a standard gradient-based regularization, yielding suboptimal performance. The detection accuracy improves rapidly as $T$ increases, peaking around $T=4$. This indicates that a short-term trajectory is sufficient to reveal the stability gap: real samples remain bounded while fakes begin to diverge. 
Further increasing $T$ leads to a slight saturation. We attributed this to the local nature of the learned potential, \ie, long trajectories may drift into poorly constrained regions far from the valid distribution.

\begin{figure}[t]
    \centering
    \includegraphics[width=0.45\textwidth]{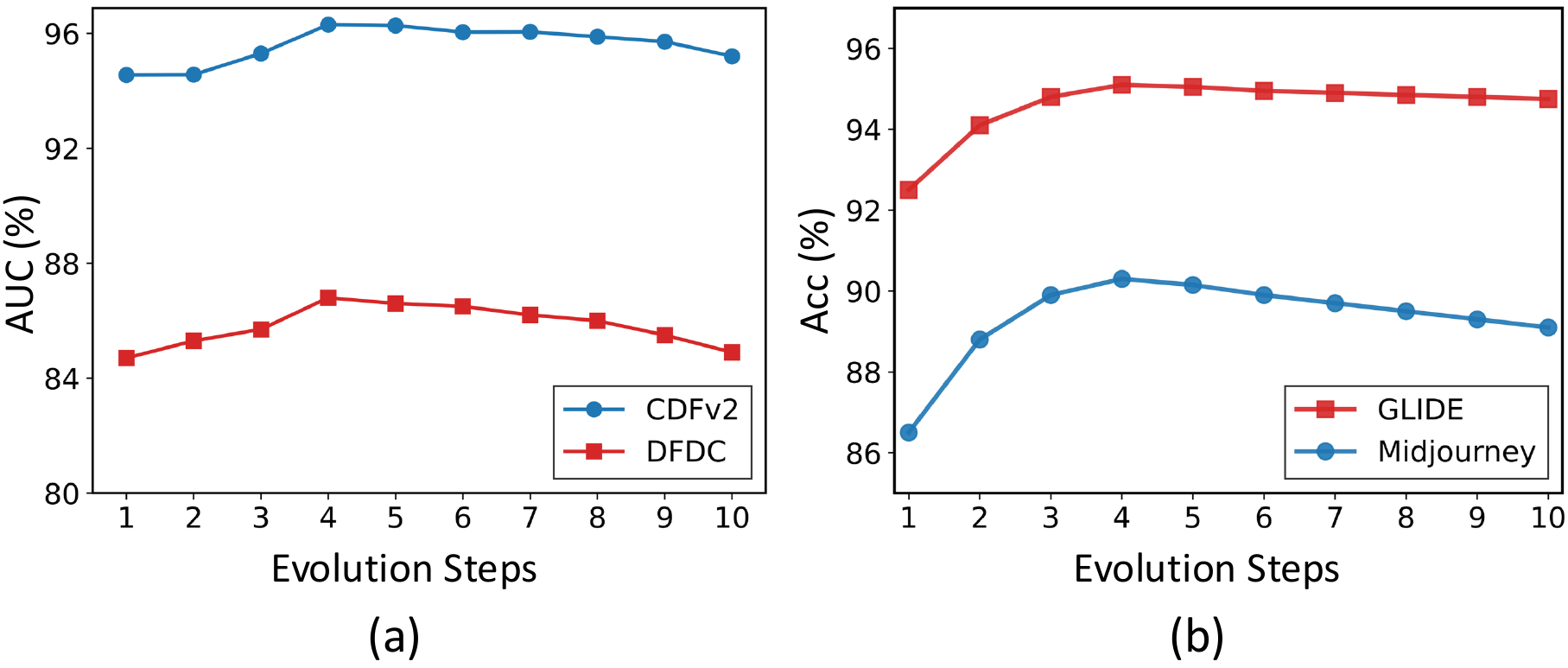}
    \vspace{-1em}
    \caption{Ablation studies on evolution steps. The detection performance for varying $T$ is reported on (a) deepfake benchmarks and (b) GenImage subsets.}
    \label{fig:sensitive_steps}
\end{figure}

\subsection{Qualitative Analysis}

\begin{figure}[t]
    \centering
    \includegraphics[width=0.45\textwidth]{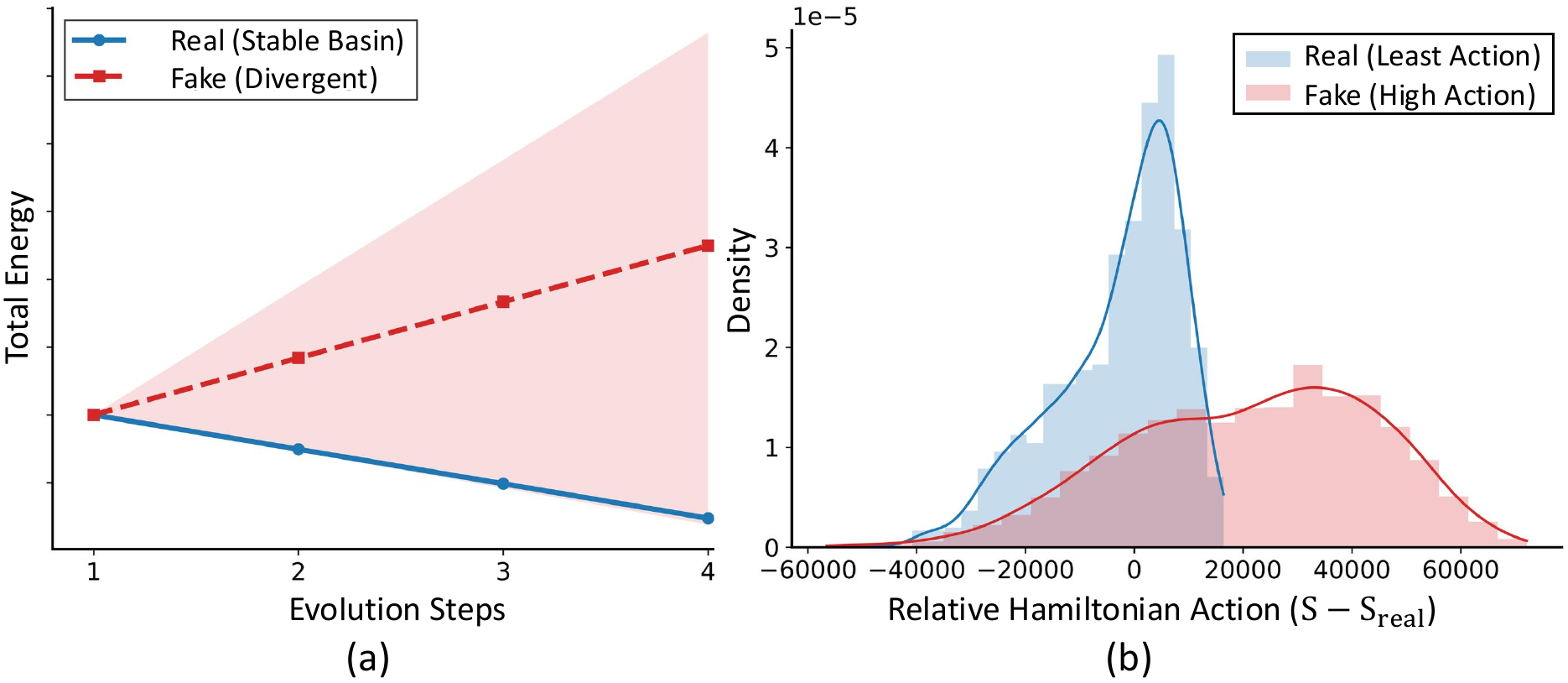}
    \caption{(a) {Energy evolution trajectories.} Real images (blue) exhibit bounded stability, whereas deepfakes (red) demonstrate a distinct divergent trend. We plotted the median normalized energy with the interquartile range (shaded). (b) {Distribution of relative Hamiltonian Action ($S-S_{real}$).} The sharp peak for real data \vs the heavy-tailed distribution for fakes highlights the discriminative power of this stability-based metric.}
    \label{fig:trag_and_invariants}
\end{figure}

\noindent \textit{\textbf{Dynamic Evolution Trajectories.}}
Figure~\ref{fig:trag_and_invariants}a illustrates the normalized energy trajectories, plotting the median value (solid line) and the interquartile range (shaded region). 
The results reveal a clear contrast: real images (blue) maintain \emph{bounded, quasi-stable orbits}, consistent with placement in deep potential basins dominated by restorative forces.
In contrast, deepfake samples (red) exhibit rapid energy divergence even within limited steps ($T=4$). This is consistent with the hypothesis that forged content occupies unstable high-potential states, where significant gradient forces act as persistent accelerators, driving the particle state away from equilibrium.

\noindent \textit{\textbf{Separability of the Action Statistic.}}
Figure~\ref{fig:trag_and_invariants}b displays the density histogram of the relative Hamiltonian Action ($S-S_{real}$), where the total area under each curve is normalized to 1.
Real samples (blue) form a sharp, narrow peak clustered tightly around zero. This is consistent with real images residing in low-energy basins with minimal fluctuation under the learned potential.
In contrast, fake images (red) exhibit a \emph{heavy-tailed distribution} spanning a wide range of large positive relative-action values. This dispersion reflects the diverse nature of generative artifacts, which place samples at various unstable elevations on the potential surface.
The distinct separation between these two distributions indicates that $S$ serves as an effective linear separator for detection.

\begin{figure}[t]
    \centering
    \includegraphics[width=0.45\textwidth]{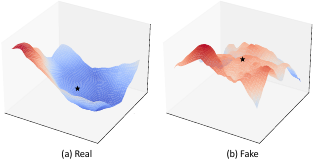}
    \caption{{Potential landscape visualization.} We visualize the learned energy surface by projecting the latent neighborhood onto random orthogonal directions. (a) Around a representative real face the learned surface forms a smooth basin. (b) Around a representative fake face the learned surface contains a high-curvature steep region associated with larger trajectory responses.}
    \label{fig:potential_landscape}
\end{figure}

\vspace{1mm}
\noindent \textit{\textbf{3D Potential Landscape Visualization.}}
To visualize the learned energy landscape, we projected the local neighborhood of samples onto random orthogonal directions.
The landscape around a real face (Figure~\ref{fig:potential_landscape}a) manifests as a \emph{smooth, convex basin}, with the data point located at a stable local minimum. This is consistent with natural images residing near equilibria under the learned potential.
In contrast, the landscape for a deepfake (Figure~\ref{fig:potential_landscape}b) reveals a \emph{high-curvature steep region} rather than a stable basin. The forged sample is located in an unstable high-potential area with large local gradients.

\subsection{Robustness and Scope Analysis}
\label{sec:robustness}
To delineate HAAD's practical scope, we evaluated two deployment-relevant axes: i) Common post-processing corruptions at test time, and ii) Backbone choice.

\begin{table}[t]
\centering
\caption{Common-corruption robustness. We reported AUC on CDF-v2. $\Delta_{\max}$ is the worst-case AUC drop relative to clean.}
\label{tab:corrupt}
    \scalebox{0.72}{
\begin{tabular}{l|c|cc|c|c|c}
\toprule \midrule
\multirow{2}{*}{\textbf{Method}} & \multirow{2}{*}{\textbf{Clean}}
    & \multicolumn{2}{c|}{\textbf{JPEG Compression}}
    & \textbf{Gaussian Blur}
    & \textbf{Gaussian Noise}
    & \multirow{2}{*}{$\Delta_{\max}\downarrow$} \\
 & & Q{=}70 & Q{=}50 & ($\sigma{=}3$) & ($\sigma^{2}{=}0.01$) & \\
\midrule
EfficientNet                      & 0.646          & 0.582          & 0.522          & 0.543          & 0.519          & 0.127 \\
SBI                         & 0.886          & 0.846          & 0.753          & 0.784          & 0.762          & 0.133 \\
Effort                               & 0.945          & 0.873          & 0.801          & 0.826          & 0.798          & 0.147 \\
\rowcolor{blue!5} \textbf{HAAD}             & \textbf{0.963} & \textbf{0.925} & \textbf{0.904} & \textbf{0.882} & \textbf{0.857} & \textbf{0.106} \\
\midrule \bottomrule
\end{tabular}%
}
\vspace{-1em}
\end{table}

\noindent\textit{\textbf{Common Post-processing Corruptions.}}
Real-world forensic pipelines re-encode, resize and re-share images before a detector sees them. To simulate this, we evaluated HAAD and three representative baselines with applying a standard set of test-time corruptions: JPEG compression at two quality levels ($Q{=}70$ and $Q{=}50$), Gaussian blur ($\sigma{=}3$), and additive Gaussian noise ($\sigma^{2}{=}0.01$). All detectors remain the same clean-trained checkpoint, without any corruption-aware retraining.
\Cref{tab:corrupt} shows that HAAD remains the strongest detector among the evaluated models under every tested corruption, and that its worst-case degradation is smaller than that observed for the baselines. This is consistent with the central framing of HAAD as a \emph{stability probe}: if the discriminative signal is the short-horizon stability of the learned potential rather than a fragile high-frequency artifact, then post-processing, which perturbs local pixel statistics more than it perturbs the smoothed manifold geometry, would be expected to degrade HAAD less than artifact-based detectors.

\vspace{1mm}
\begin{table}[t]
\centering
\caption{Backbone sensitivity evaluation. We reported the AUC scores of two different backbones on CDF-v2.}
\label{tab:backbone}
    \scalebox{0.90}{
\begin{tabular}{l|cc|c}
\toprule \midrule
\textbf{Backbone} & \textbf{Base} & \textbf{+HAAD} & $\Delta$ \\
\midrule
ResNet-50~\cite{ResNet}              & 0.683 & \textbf{0.806} & +0.123 \\
ViT-B/16                    & 0.915 & \textbf{0.948} & +0.033 \\
\midrule \bottomrule
\end{tabular}%
}
\vspace{-1em}
\end{table}

\noindent\textit{\textbf{Backbone Sensitivity.}}
To evaluate whether the Hamiltonian stability probe generalizes across architectures, we applied the HAAD head to two additional backbones: ResNet-50~\cite{ResNet} and ViT-B/16.
As shown in \Cref{tab:backbone}, HAAD improves every backbone. For instance, it improves the ResNet-50 baseline by 12.3\,pp. This indicates that $S$ and $D$ serve as an effective detection prior across diverse backbone architectures.

\section{Conclusion and Discussion}
\label{sec:conc}
In this paper, we propose a new perspective on deepfake detection: moving from static pattern recognition to \textbf{dynamical stability analysis}.
By conceptualizing the latent features as a physical energy landscape, we propose \textbf{HAAD}, a framework that complements artifact-based detection by probing dynamical stability across generative architectures.
Our core contribution is the operationalization of the Manifold Stability Hypothesis via a symplectic evolution scheme.
Furthermore, the short-horizon symplectic evolution amplifies microscopic structural irregularities into measurable trajectory differences, summarized by the Hamiltonian Action and Energy Dissipation.
Experimental results demonstrate that the proposed stability-based prior generalizes across generators and datasets, outperforming several SoTA baselines.

\noindent\textbf{Limitations and Societal Considerations.}
HAAD is a stability-based probe and its effectiveness depends on the learned potential capturing structural regularities that current generators do not explicitly enforce. The Manifold Stability Hypothesis is distributional and does not provide per-sample guarantees. Like all published detectors, HAAD participates in the ongoing detector-generator arms race: future generators may learn to suppress the measured stability cues, which motivates continued development of physics-grounded detection priors. Because forensic false positives can carry reputational and legal consequences, we view HAAD as a decision-support signal that should be used alongside provenance analysis, contextual evidence, and human review.

\bibliographystyle{IEEEtran}
\bibliography{__Main}

\clearpage
\appendices
\section{Theoretical Motivation for the Manifold Stability Hypothesis}
\label{app:theory}

In this section, we provide a theoretical motivation for the Manifold Stability Hypothesis proposed in \Cref{sec:Theory}. We connect the Principle of Least Action (PLA) to the static geometric properties of the image manifold, motivating the intuition that real and fake data may occupy distinct topological states (valleys \vs peaks). We emphasize that the following derivation is conditional on assumptions about image formation and does not constitute a universal proof.

\subsection{From Least Action to Minimum Potential Energy}

The \textbf{PLA} states that the physical evolution of a system between two states follows a path that makes the action functional stationary ($\delta \mathcal{S} = 0$). While PLA primarily governs dynamic trajectories, it has a direct corollary for the static states of macroscopic systems formed under natural laws: the \textbf{Principle of Minimum Potential Energy}.

Natural images can be viewed as terminal states of complex physical processes (\eg, photon transport, sensor integration, and chemical development). These processes are generally \textbf{dissipative}, \ie, they involve energy exchange with the environment and tend toward thermodynamic equilibrium.

\vspace{0.5em}
\noindent \textbf{Remark (Relaxation to Equilibrium via Langevin Dynamics).}
Consider the image formation process as a stochastic dynamical system governed by Langevin dynamics:
\begin{equation}
    \frac{d\mathbf{q}}{dt} = -\nabla V(\mathbf{q}) - \gamma \mathbf{p} + \sqrt{2\gamma \beta^{-1}} \boldsymbol{\xi}(t),
\end{equation}
where $\gamma$ represents the dissipation (friction) coefficient, and $\boldsymbol{\xi}(t)$ is Gaussian noise. Over time, dissipative terms drain the kinetic energy, and the system naturally settles into states where the net force vanishes ($\nabla V \approx 0$) and the potential energy is locally minimized.
\vspace{0.5em}

Therefore, we hypothesize that \textbf{real images ($\mathbf{x}_{real}$)} tend to represent these `relaxed' states. They are expected to reside near the \textbf{basins of attraction} (local minima) of the underlying natural potential surface $V$, characterized by structural coherence and minimal internal tension.

\subsection{Deepfakes as High-Tension States}

In contrast, deepfake generation (\eg, GANs, Diffusion Models) is not a natural dissipative relaxation but a \textbf{statistical optimization}. The generator $G$ aims to map a latent code $\mathbf{z}$ to a target distribution $p_{data}$ by minimizing a statistical divergence $\mathcal{D}$ (\eg, Jensen-Shannon or KL divergence).

However, minimizing statistical divergence $\mathcal{D}$ does not imply minimizing the physical potential $V$.
\begin{itemize}
    \item \textbf{Forced Fitting:} To satisfy the semantic constraints (\eg, `identity of Person A' + `expression of Person B'), the generator must fuse disjoint features that may not naturally coexist.
    \item \textbf{Structural Tension:} This fusion creates microscopic inconsistencies, such as unnatural pixel correlations or discontinuity in high-frequency domains. In our Hamiltonian framework, these inconsistencies manifest as \textbf{structural tension} (high potential energy).
\end{itemize}

\noindent \textbf{Physical Analogy.} Conceptually, a deepfake can be likened to a \textbf{compressed spring}: an external optimization pins the sample to a specific point, while the learned potential's gradient $-\nabla V$ acts as a restorative force pulling against that placement. The system therefore has high stored potential energy. Unlike real images, which rest in basins of the potential well, deepfakes under this view are suspended on \textbf{steep slopes} of $V$, where the gradient magnitude is significant: $\|\nabla V\| \gg 0$.

\subsection{Hamiltonian Dynamics as a Stability Probe}

Since the potential function $V$ is high-dimensional and implicit, we cannot directly measure the `elevation' of a sample. Instead, we use \textbf{Hamiltonian Dynamics} as a probe to test the local stability of the state.

We treat the feature vector $\mathbf{q}$ as a particle initialized at rest; the potential gradient then induces virtual momentum during the rollout. The subsequent evolution follows the Hamiltonian $\mathcal{H} = T + V$:
\begin{equation}
    \frac{d\mathbf{p}}{dt} = -\nabla V(\mathbf{q}).
\end{equation}
This leads to two distinct dynamic behaviors:
\begin{enumerate}
    \item \textbf{Oscillation (Real):} For a particle at a local minimum ($\nabla V \approx 0$), the force acts as a restorative force, keeping the particle confined within the basin. The kinetic energy $T$ remains bounded and small. Consequently, the accumulated \textbf{Hamiltonian Action (Cost)} remains close to the real-sample baseline (equivalently, $S-S_{real} \approx 0$ in the local comparison).
    \item \textbf{Divergence (Fake):} For a particle on a slope ($\|\nabla V\| \gg 0$), the non-zero gradient acts as a persistent accelerating force. The stored potential energy is rapidly converted into kinetic energy ($T \propto t^2$), causing the particle to accelerate away from the initial state. This results in a larger accumulated \textbf{Hamiltonian Action}.
\end{enumerate}

By simulating this short-term evolution, our HAAD framework effectively amplifies the microscopic static tension into a macroscopic dynamic signal, providing a complementary discriminator that is less dependent on specific visual artifacts than static detection methods.

\section{Proofs and Derivations for Proposition~\ref{prop:divergence}}
\label{app:proofs}

In this section, we provide the detailed derivation for Proposition~\ref{prop:divergence} (Divergence of Hamiltonian Action) presented in the main text. We analyze the short-term evolution of the latent state $\mathbf{q}$ under the proposed Hamiltonian system and bound the growth of the Hamiltonian Action $S$ for real and fake samples, respectively. This derivation is conditional on the Manifold Stability Hypothesis (Assumption~3.1); it formalizes the consequence of the hypothesis but does not prove that the hypothesis holds universally.

\subsection{Preliminaries}
Recall the canonical Hamiltonian of our system (Definition~\ref{def:hamilton_system} in the main text):
\begin{equation}
    \mathcal{H}(\mathbf{q}, \mathbf{p}) = \frac{1}{2} \|\mathbf{p}\|^2 + V(\mathbf{q}).
\end{equation}
The dynamics follow Hamilton's equations:
\begin{equation}
    \dot{\mathbf{q}} = \mathbf{p}, \quad \dot{\mathbf{p}} = -\nabla V(\mathbf{q}).
\end{equation}
In the implementation (\Cref{sec:symplectic}), the position update is additionally preconditioned by a state-conditioned diagonal matrix $\mathbf{M}(\mathbf{q})^{-1}$ that rescales per-feature step sizes; this preconditioner modulates the constant in the kinetic-growth bound below but does not affect the $\mathcal{O}(\tau^2)$ separation between real and fake samples. We therefore analyze the canonical system here for clarity. We consider a short time interval $t \in [0, \tau]$, initialized at rest: $\mathbf{q}(0) = \mathbf{q}_0$ and $\mathbf{p}(0) = \mathbf{0}$.

\subsection{Proof of Divergence}
To analyze the local behavior, we perform a second-order Taylor expansion of the potential energy $V(\mathbf{q})$ around the initial state $\mathbf{q}_0$. For any state $\mathbf{q}$ in the local neighborhood:
\begin{equation}
\begin{aligned}
    V(\mathbf{q}) \approx{}& V(\mathbf{q}_0)
    + \nabla V(\mathbf{q}_0)^\top (\mathbf{q} - \mathbf{q}_0) \\
    &+ \frac{1}{2}(\mathbf{q} - \mathbf{q}_0)^\top
    \mathbf{H}(\mathbf{q}_0)(\mathbf{q} - \mathbf{q}_0),
\end{aligned}
\end{equation}
where $\mathbf{H}(\mathbf{q}_0)$ is the Hessian matrix of $V$ at $\mathbf{q}_0$.

\subsubsection{Case 1: Real Samples (Stable Equilibrium)}
Under the \textbf{Manifold Stability Hypothesis} (Assumption~\ref{ass:stability}), a real sample $\mathbf{q}_{real}$ is hypothesized to reside near a local minimum. This neighborhood implies two conditions:
\begin{enumerate}
    \item \textbf{Vanishing Gradient:} $\|\nabla V(\mathbf{q}_{real})\| \approx 0$.
    \item \textbf{Positive Definite Hessian:} The Hessian $\mathbf{H}(\mathbf{q}_{real})$ is positive semi-definite (convex basin).
\end{enumerate}
Substituting $\nabla V \approx \mathbf{0}$ into Hamilton's equations:
\begin{equation}
    \dot{\mathbf{p}}(0) = -\nabla V(\mathbf{q}_{real}) \approx \mathbf{0}.
\end{equation}
Since the initial momentum $\mathbf{p}(0) = \mathbf{0}$ and the initial force is zero, the system remains in equilibrium at first-order approximation. The kinetic energy $T(t)$ remains negligible:
\begin{equation}
    T(t) = \frac{1}{2} \|\mathbf{p}(t)\|^2 \approx 0.
\end{equation}
Consequently, the accumulated Hamiltonian Action $S$ is dominated by the local baseline potential of the real sample:
\begin{equation}
    S_{real} \approx \frac{1}{\tau} \int_{0}^{\tau} V(\mathbf{q}_{real}) dt \approx V(\mathbf{q}_{real}).
\end{equation}

\subsubsection{Case 2: Fake Samples (Unstable State)}
For a deepfake sample $\mathbf{q}_{fake}$, the hypothesis posits that it is more likely to lie on a slope with a significant non-zero gradient in the local comparison of interest. Let $\mathbf{g} = \nabla V(\mathbf{q}_{fake})$ be the gradient vector, where $\|\mathbf{g}\| = C \gg 0$.

For a small time step $t$, we can assume the gradient force is approximately constant (zeroth-order approximation of force). The evolution of momentum is given by:
\begin{equation}
    \mathbf{p}(t) = \mathbf{p}(0) + \int_{0}^{t} -\nabla V(\mathbf{q}(s)) ds \approx -\mathbf{g} t.
\end{equation}

Under the canonical kinetic energy, $T(t)$ grows quadratically with time:
\begin{equation}
    T(t) = \frac{1}{2} \|\mathbf{p}(t)\|^2 = \frac{1}{2} \|\mathbf{g}\|^2 t^2 = \frac{1}{2} C^2 t^2.
\end{equation}
This confirms the quadratic surge in kinetic energy.

Now, consider the Hamiltonian Action $S$. Unlike the Lagrangian ($T-V$), we define $S$ as the accumulation of total energy magnitude $\mathcal{H} = T + V$ to maximize signal detection.
\begin{equation}
    S_{fake} = \frac{1}{\tau} \int_{0}^{\tau} (T(t) + V(\mathbf{q}(t))) dt.
\end{equation}
Since $\mathbf{q}_{real}$ is hypothesized to lie near a local minimum while $\mathbf{q}_{fake}$ lies on a higher-gradient slope, we consider the local comparison in which the potential gap $\Delta V = V(\mathbf{q}_{fake}) - V(\mathbf{q}_{real})$ is positive, and $T(t)$ adds a positive quadratic contribution. Specifically, integrating the kinetic term yields:

\begin{equation}
    \int_{0}^{\tau} T(t) dt \approx \int_{0}^{\tau} \frac{1}{2} C^2 t^2 dt = \frac{1}{6} C^2 \tau^3.
\end{equation}
Therefore, the Action gap scales as:
\begin{equation}
    S_{fake} - S_{real} \approx \Delta V + \frac{\tau^2}{6} C^2.
\end{equation}
Comparing Case 1 and Case 2, the positive potential gap together with the additional $\mathcal{O}(\tau^2)$ kinetic term implies:
\begin{equation}
    S_{fake} - S_{real} > 0.
\end{equation}
This completes the proof. \qed

\section{Volume Preservation of the Symplectic Euler Update (Constant-Mass Case)}
\label{app:symplectic_proof}

A desirable property for our HAAD framework is that the measured Energy Dissipation ($D$) primarily reflects the geometric properties of the data manifold rather than numerical artifacts of the integration scheme. In this section, we show that the \textbf{Symplectic Euler} integrator, \emph{in the ideal constant-mass case}, defines a \textit{symplectomorphism}, meaning it strictly preserves the phase-space volume (Liouville's Theorem).

\noindent\textbf{Scope of this result.} The analysis below treats $\mathbf{M}$ as a constant (state-independent) positive-definite diagonal matrix. In the main-text implementation (\Cref{sec:symplectic}) we instead use a \emph{state-conditioned} preconditioner $\mathbf{M}(\mathbf{q})^{-1}$ in the position update and explicitly omit the $\mathbf{q}$-derivative terms that would appear in an exact variable-mass Hamiltonian flow. Consequently, the constant-mass proof below should be read as motivation for our choice of integrator family (semi-implicit Euler with momentum-before-position ordering) rather than as a strict symplectomorphism guarantee for the implemented preconditioner-augmented update. We accordingly label the implemented rollout ``Hamiltonian-inspired, symplectic-style'' in \Cref{sec:symplectic} and reserve the strict ``symplectomorphism'' term for the constant-mass analysis here. We also note upfront that volume preservation does not imply exact Hamiltonian conservation, so $D$ is not fully isolated from numerical effects even in the constant-mass case.

\subsection{Jacobian Analysis of the Discrete Map}
Let $\Phi: (\mathbf{q}_t, \mathbf{p}_t) \mapsto (\mathbf{q}_{t+1}, \mathbf{p}_{t+1})$ denote the discrete update map defined by the Symplectic Euler scheme (\Cref{eq:update_p} and (\ref{eq:update_q}) in the main text). For a system with $d$ dimensions, the update rule is:
\begin{align}
    \mathbf{p}_{t+1} &= \mathbf{p}_t - \eta \nabla V(\mathbf{q}_t), \\
    \mathbf{q}_{t+1} &= \mathbf{q}_t + \eta \mathbf{M}^{-1} \mathbf{p}_{t+1}.
\end{align}
To prove volume preservation, we must show that the determinant of the Jacobian matrix of this transformation, $J = \frac{\partial(\mathbf{q}_{t+1}, \mathbf{p}_{t+1})}{\partial(\mathbf{q}_t, \mathbf{p}_t)}$, is identically equal to 1.

The transformation can be decomposed into two sequential sub-steps (shear transformations):
i) Momentum Update (Kick): $(\mathbf{q}_t, \mathbf{p}_t) \to (\mathbf{q}_t, \mathbf{p}_{t+1})$
ii) Position Update (Drift): $(\mathbf{q}_t, \mathbf{p}_{t+1}) \to (\mathbf{q}_{t+1}, \mathbf{p}_{t+1})$

\noindent \textbf{Step i): Jacobian of Momentum Update.}
Differentiating $\mathbf{p}_{t+1}$ with respect to the input state $(\mathbf{q}_t, \mathbf{p}_t)$:
\begin{equation}
    J_1 = \begin{pmatrix}
    \frac{\partial \mathbf{q}_t}{\partial \mathbf{q}_t} & \frac{\partial \mathbf{q}_t}{\partial \mathbf{p}_t} \\
    \frac{\partial \mathbf{p}_{t+1}}{\partial \mathbf{q}_t} & \frac{\partial \mathbf{p}_{t+1}}{\partial \mathbf{p}_t}
    \end{pmatrix} = \begin{pmatrix}
    \mathbf{I} & \mathbf{0} \\
    -\eta \mathbf{H}(\mathbf{q}_t) & \mathbf{I}
    \end{pmatrix},
\end{equation}
where $\mathbf{H}(\mathbf{q}_t)$ is the Hessian matrix of potential $V$. Since this is a lower triangular block matrix with identity blocks on the diagonal, its determinant is:
\begin{equation}
    \det(J_1) = \det(\mathbf{I}) \cdot \det(\mathbf{I}) = 1.
\end{equation}

\noindent \textbf{Step ii): Jacobian of Position Update.}
Differentiating $\mathbf{q}_{t+1}$ with respect to the intermediate state $(\mathbf{q}_t, \mathbf{p}_{t+1})$:
\begin{equation}
    J_2 = \begin{pmatrix}
    \frac{\partial \mathbf{q}_{t+1}}{\partial \mathbf{q}_t} & \frac{\partial \mathbf{q}_{t+1}}{\partial \mathbf{p}_{t+1}} \\
    \frac{\partial \mathbf{p}_{t+1}}{\partial \mathbf{q}_t} & \frac{\partial \mathbf{p}_{t+1}}{\partial \mathbf{p}_{t+1}}
    \end{pmatrix} = \begin{pmatrix}
    \mathbf{I} & \eta \mathbf{M}^{-1} \\
    \mathbf{0} & \mathbf{I}
    \end{pmatrix}.
\end{equation}
This is an upper triangular block matrix. Similarly, its determinant is:
\begin{equation}
    \det(J_2) = 1.
\end{equation}

\noindent \textbf{Total Jacobian.}
By the chain rule, the Jacobian of the full step $\Phi$ is the product $J = J_2 J_1$. The determinant is:
\begin{equation}
    \det(J) = \det(J_2) \det(J_1) = 1 \cdot 1 = 1.
\end{equation}

\subsection{Physical Implication: Trustworthiness of Dissipation}
Since $\det(J) = 1$, the Symplectic Euler integrator preserves the phase-space volume element $d\mathbf{q} \wedge d\mathbf{p}$ exactly.
\begin{remark}[\textbf{Reduced Numerical Damping}]
Unlike standard Euler or Runge-Kutta methods, where $|\det(J)| \neq 1$ leads to artificial energy drift (numerical damping or explosion), symplectic methods better control long-term energy drift and can often be interpreted as approximately conserving a nearby modified Hamiltonian under suitable step-size assumptions.
\end{remark}

This mathematical property means that the symplectic integrator reduces solver-induced energy drift relative to non-symplectic methods. However, volume preservation does not imply exact Hamiltonian conservation, so the observed Energy Dissipation ($D > 0$) is not guaranteed to be exclusively data-induced. The correct interpretation is that $D$ is an \textbf{empirically useful discriminative signal}: symplectic integration suppresses some numerical artifacts, making $D$ more attributable to the geometry of the potential surface, but it does not provide a strict isolation guarantee. The empirical effectiveness of $D$ as a detection cue is validated through the ablation study in the main text.

\section{Additional Experiments and Hyperparameter Sensitivity}
\label{app:experiments}

\subsection{Hyperparameter Sensitivity Analysis}

Besides the physical regularization weight $\lambda_{phy}$ in the main text, we further investigate the sensitivity of the proposed framework to two key hyperparameters: the Hamiltonian evolution step size $\eta$, and the energy margin $\gamma$. All experiments are conducted on the cross-dataset setting (Train on FF++, Evaluate on Celeb-DF-v2).

\noindent\textbf{i) Impact of Step Size ($\eta$).}
The step size $\eta$ determines the discrete time interval for the symplectic integrator. Its effect on detection performance is shown in \Cref{fig:sens_eta}.
\begin{itemize}
    \item \textbf{Too Small ($\eta < 0.1$):} The system evolves too slowly. Within the fixed steps $T=4$, the particle hardly moves from the initial state, resulting in a weak detection signal ($S \approx 0$ for both real and fake).
    \item \textbf{Optimal ($\eta \approx 0.4$):} Provides sufficient `kick' to reveal the instability of deepfakes while maintaining numerical stability for real samples.
    \item \textbf{Too Large ($\eta > 0.8$):} Leads to numerical instability and gradient explosion, causing the training to diverge.
\end{itemize}

\begin{figure}[h]
    \centering
    \includegraphics[width=0.48\textwidth]{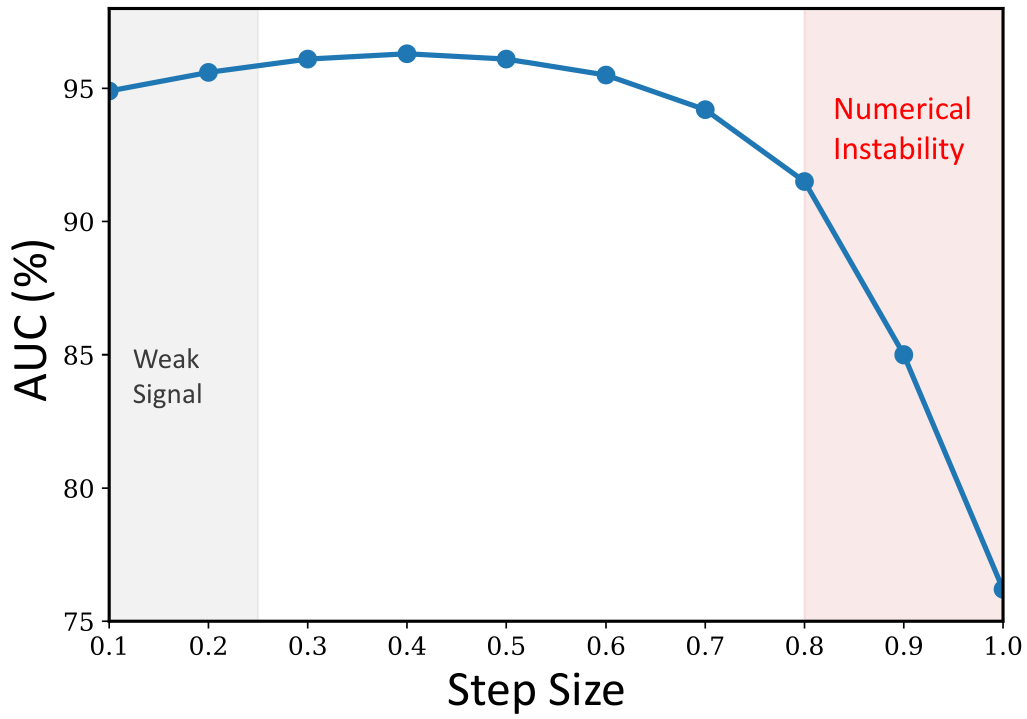}
    \caption{\textbf{Sensitivity to Step Size $\eta$.} Performance peaks around $\eta=0.4$; excessively large steps lead to numerical instability.}
    \label{fig:sens_eta}
\end{figure}

\subsection{Solver-Induced vs.\ Data-Induced Components of $D$}
\label{app:solver_decomp}

Because Symplectic Euler conserves a modified (shadow) Hamiltonian with $\mathcal{O}(\eta^2)$ error per step~\cite{leimkuhler2004simulating}, its baseline drift for near-conservative trajectories (real samples) grows predictably with step size $\eta$. In contrast, fake-induced drift emerges from large potential gradients and grows faster for larger $\|\nabla V\|$. A complete disentanglement would require measuring real-sample $D$ across multiple $\eta$ values to isolate the $\mathcal{O}(\eta^2)$ solver baseline, then comparing to fake-sample $D$ at the same $\eta$. The integrator ablation in \Cref{tab:abl_solver} provides indirect evidence for this separation: the Euler method, whose $\mathcal{O}(\eta)$ drift raises the solver noise floor, degrades detection performance by blurring the real/fake gap in $D$, whereas Symplectic Euler's more predictable drift preserves it.

\subsection{Computational Efficiency Analysis}

A key concern for physics-based deep learning is the inference latency. In Table~\ref{tab:inference_time}, we provide a detailed breakdown of the computational cost.
The experiments are conducted on a single NVIDIA H100 GPU with a batch size of 32.

\begin{table}[h]
\centering
\caption{Detailed comparison of computational cost (Inference per image).}
\label{tab:inference_time}
\begin{tabular}{l|c|c|c}
\toprule
\textbf{Method} & \textbf{Steps $T$} & \textbf{Time (ms)} & \textbf{Memory (MB)} \\
\midrule
CLIP ViT-L/14 (backbone) & - & 11.5 & 2450 \\
\midrule
HAAD (RK4) & 4 & 45.0 & 3100 \\
HAAD (Euler) & 4 & 12.0 & 2480 \\
\textbf{HAAD (Symplectic Euler)} & \textbf{4} & \textbf{14.0} & \textbf{2485} \\
\bottomrule
\end{tabular}
\end{table}

\textbf{Analysis:}
\begin{itemize}
    \item \textbf{Runge-Kutta (RK4):} While precise, it requires 4 gradient evaluations per step, resulting in a latency of \textbf{45ms}, which is $\approx 4.3\times$ slower than the backbone. This aligns with the findings in Table 5.
    \item \textbf{Symplectic Euler (Ours):} By requiring only 1 gradient evaluation per step, it achieves an inference time of \textbf{14ms}. The slight overhead compared to standard Euler (12ms) stems from the computation of the learnable preconditioner $\mathbf{M}(\mathbf{q})^{-1}$, which provides feature-specific sensitivity.
    \item \textbf{Memory Efficiency:} The Hamiltonian evolution operates in the low-dimensional latent space ($N \times 64$), inducing a negligible memory increase ($\approx 35$ MB) over the backbone.
\end{itemize}

\section{Detailed Network Architectures and Implementation}
\label{app:implementation}

In this section, we provide the specific architectural details of the proposed HAAD framework, including the backbone configuration, the structure of the physical projection heads, the graph topology construction, and the hyperparameter settings.

\subsection{HAAD Module Architecture}
The HAAD module projects high-dimensional semantic features into a low-dimensional physical state space. The specific components are defined as follows:

\subsubsection{Physical State Projection}
To reduce computational complexity and enforce a bottleneck, we project the input features $\mathbf{x}$ to a physical state $\mathbf{q} \in \mathbb{R}^{N \times D_{phy}}$ with $D_{phy} = 64$.
\begin{equation}
    \mathbf{q} = \text{Linear}(D_{in} \to D_{phy})(\mathbf{x}).
\end{equation}

\subsubsection{Mass Estimation Network}
The state-conditioned diagonal preconditioner $\mathbf{M}^{-1}(\mathbf{q})$ rescales the position update for each patch and feature dimension. It is estimated via a lightweight MLP with a Softplus activation to keep the scaling positive for numerical stability:
\begin{equation}
    \mathbf{M}^{-1}(\mathbf{q}) = \text{Softplus}\left( \text{MLP}(\mathbf{q}) \right) + \epsilon,
\end{equation}
where $\epsilon = 10^{-3}$. The MLP architecture is:
\begin{itemize}
    \item Layer 1: Linear($64 \to 64$) + ReLU
    \item Layer 2: Linear($64 \to 1$)
\end{itemize}
The output is broadcasted to match the dimension of $\mathbf{q}$ for element-wise multiplication.

\subsubsection{Potential Parameterization Heads}
To compute the photometric potential $V_{photo}$, we estimate intrinsic physical properties using shallow linear heads:
\begin{itemize}
    \item \textbf{Surface Normal ($\mathbf{n}$):} Linear($64 \to 3$) followed by $L_2$ normalization.
    \item \textbf{Albedo ($\rho$):} Linear($64 \to 1$) followed by Sigmoid activation.
    \item \textbf{Global Light ($\mathbf{l}$):} Linear($64 \to 3$), averaged over all patches to obtain a global vector.
\end{itemize}

\subsection{Graph Topology Construction}
To compute the geometric potential $V_{geo}$, we construct a spatial graph $\mathcal{G}$ over the patch grid. Unlike fully connected graphs, we utilize a local k-NN graph to capture local structural continuity.
\begin{itemize}
    \item \textbf{Coordinates:} We assign 2D grid coordinates $(y_i, x_i)$ to each patch.
    \item \textbf{Connectivity:} For each patch, we connect it to its $k=8$ nearest spatial neighbors.
    \item \textbf{Weights:} The edge weights $w_{ij}$ are computed using a Gaussian kernel based on Euclidean distance $d_{ij}$:
    \begin{equation}
        w_{ij} = \exp\left( -\frac{d_{ij}^2}{2\sigma^2} \right), \quad \text{with } \sigma = 8.0.
    \end{equation}
    \item \textbf{Laplacian:} We compute the sparse Graph Laplacian $\mathbf{L} = \mathbf{D} - \mathbf{W}$, where $\mathbf{D}$ is the degree matrix.
\end{itemize}

\subsection{Symplectic Integration Scheme}
We utilize the Symplectic Euler method for the discrete evolution of the system. Let $\eta$ be the step size and $T$ be the number of steps. The update rule at step $t$ is implemented as:
\begin{align}
    \mathbf{F}_t &= -\nabla_{\mathbf{q}} V(\mathbf{q}_t), \\
    \mathbf{p}_{t+1} &= \mathbf{p}_t + \eta \cdot \mathbf{F}_t, \\
    \mathbf{q}_{t+1} &= \mathbf{q}_t + \eta \cdot (\mathbf{M}^{-1} \odot \mathbf{p}_{t+1}).
\end{align}

This implementation is the same semi-implicit momentum-before-position update used in the main text, with $\mathbf{F}_t=-\nabla_{\mathbf{q}}V(\mathbf{q}_t)$ and the state-conditioned preconditioner applied only in the position update.

\end{document}